%% file: main.tex
\documentclass{article}

\usepackage{iclr2027_conference,times}
\usepackage{amsthm}
\usepackage{amsmath,amssymb,mathtools}
\usepackage{amsthm}
\usepackage{booktabs,multirow,array,tabularx}
\usepackage{graphicx}
\usepackage{xcolor}
\usepackage{microtype}
\usepackage[htt]{hyphenat} %
\usepackage{algorithm}
\usepackage[noend]{algpseudocode}
\usepackage{enumitem}
\usepackage{url}
\usepackage{xurl}
\usepackage{hyperref}
\usepackage[nameinlink,capitalize,noabbrev]{cleveref}
\usepackage{caption}
\usepackage{subcaption}
\usepackage{amsmath,amssymb,amsthm}

\theoremstyle{plain}
\newtheorem{wmpatheorem}{Theorem}[section]

\theoremstyle{definition}
\newtheorem{wmpaassumption}[wmpatheorem]{Assumption}

\theoremstyle{plain}
\hypersetup{colorlinks=true,citecolor=blue,linkcolor=blue,urlcolor=blue}
\setlist[itemize]{leftmargin=*,topsep=2pt,itemsep=1pt}
\setlist[enumerate]{leftmargin=*,topsep=2pt,itemsep=1pt}

\newcommand{\wmpa}{\textsc{WMPA}}
\providecommand{\wmpp}{\wmpa}
\newcommand{\bestfixed}{\textsc{Best-Policy}}
\newcommand{\randomswitch}{\textsc{Random-Switch}}

\newcommand{\E}{\mathbb{E}}
\newcommand{\PiBank}{\boldsymbol{\Pi}}
\newcommand{\calD}{\mathcal{D}}
\newcommand{\calS}{\mathcal{S}}
\newcommand{\calA}{\mathcal{A}}
\newcommand{\calG}{\mathcal{G}}

\input{tables/numbers}

\IfFileExists{tables/numbers_ablations.tex}{\input{tables/numbers_ablations}}{}
\IfFileExists{tables/numbers_mechanism.tex}{

\input{tables/numbers_mechanism}}{}
\IfFileExists{tables/numbers_diagnostics.tex}{

\input{tables/numbers_diagnostics}}{}
\IfFileExists{tables/numbers_wmseeds.tex}{

\input{tables/numbers_wmseeds}}{}
\IfFileExists{tables/numbers_kcgrid.tex}{

\input{tables/numbers_kcgrid}}{}
\IfFileExists{tables/numbers_vhsel.tex}{

\input{tables/numbers_vhsel}}{}
\IfFileExists{tables/numbers_native.tex}{

\input{tables/numbers_native}}{}

\title{World-Model Policy Arbiter for\\
Goal-Conditioned Reinforcement Learning}

\iclrfinalcopy

\author{Junwei Quan\thanks{Corresponding author: \texttt{junwei.quan@mail.utoronto.ca}} \\
University of Toronto \\
\And
Evgenii Opryshko \\
University of Toronto, Vector Institute \\
\And
Nicholas Rhinehart \\
University of Toronto \\
\And
Igor Gilitschenski \\
University of Toronto, Vector Institute \\
}
\date{}

\begin{document}
\maketitle
\lhead{Preprint. Under review.}
\begin{abstract}
Offline goal-conditioned reinforcement learning (GCRL) has produced a diverse set of goal-reaching algorithms, yet no single algorithm performs best across environments, goals, and even different phases of the same task.
Rather than deploying only the best-performing policy, we ask whether a set of frozen goal-conditioned policies can be used collectively as a \emph{portfolio}, deciding at every state which policy should act.
Choosing a policy at each state is not straightforward. The policies' own value functions cannot be compared directly: they may use different scales, and some policies have no value function. We need to judge each policy by the states it is likely to reach, even though we can execute only one policy at a time. We also need to avoid switching so often that control becomes unstable.
To address these challenges, we introduce World-Model Policy Arbiter (\wmpa), a test-time framework that, given a set of frozen policies as input, rolls out each frozen policy in a learned state-space world model, evaluates the imagined futures with a shared goal-conditioned value function, and executes the highest-scoring policy for a short commitment interval before the next round of arbitration (policy selection).
\wmpa{} assumes access to a bank of frozen goal-conditioned policies and requires neither policy retraining nor privileged task-specific knowledge. Under the official OGBench evaluation protocol on \NumEnvs{} state-based datasets spanning maze navigation as well as cube, scene, and puzzle manipulation, \wmpa{} improves the macro-average success rate from the \MeanBest\% achieved by the best policy selected per dataset to \MeanWMPP\%, with statistically significant gains on \NumSigUp{} datasets. These gains include \GainCubeDoublePlay{} percentage points on \texttt{cube-double-play} and \GainScenePlay{} percentage points on \texttt{scene-play}. Project website: \url{https://junwei0102.github.io/wmpa-webpage/}
\end{abstract}

\section{Introduction}
\label{sec:intro}

Goal-conditioned reinforcement learning (GCRL) seeks policies that reach user-specified goals, many of which require a sequence of qualitatively different phases: to stack two cubes at a target location, a robot arm must bring its free gripper to the first cube, grasp and place it, then repeat for the second cube.
GCRL methods learn such policies through behavior cloning, temporal-difference learning, contrastive learning, quasimetric learning, or hierarchical control \citep{eysenbach2022contrastive,wang2023quasimetric,park2023hiql}, and because their assumptions differ, no single method dominates across environments, goals, and task horizons \citep{park2025ogbench}.
Nor does one method dominate across the phases of a single task: on OGBench's two-cube tasks (\texttt{cube-double-play}), GCIQL, the strongest of six standard GCRL methods, succeeds in 36\% of episodes and HIQL in only 5\%, yet our method, selecting among the six policies within each episode, succeeds in 69\% (\cref{tab:main}).
In the stacking episode of \cref{fig:episode}, for instance, control alternates between HIQL, which moves the free gripper to the next cube, and GCIQL, which grasps and places it.
Independently trained policies can thus act as pieces of a larger solution, and we study how to string them together at test time without retraining any of them.

Most prior work aims to train a single stronger policy, and the standard workflow for practitioners is to train several candidates, keep the one with the highest average success, and discard the rest.
In the two-cube example, this keeps GCIQL and discards HIQL, the policy that moves the gripper between cubes in \cref{fig:episode}.
Keeping all of them instead raises a question: \emph{given a set of frozen goal-conditioned policies, which policy should act from the current state for the current goal?}

\begin{figure}[t]
    \centering
    \includegraphics[width=\linewidth]{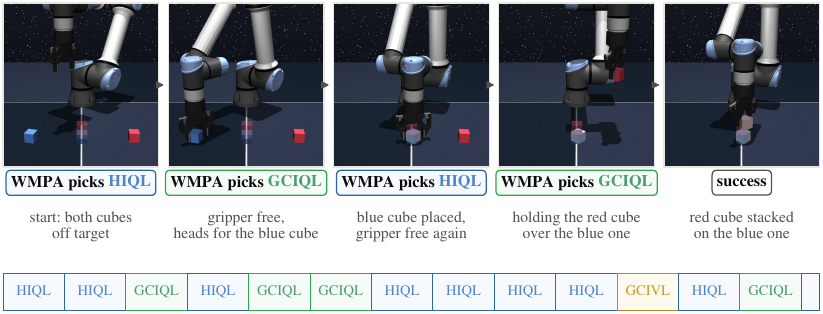}
    \caption{\textbf{\wmpa{} composes policies within one episode.} A \texttt{cube-double-play} evaluation episode. Top: the real state at four arbitration steps and at success, with the policy the world model selects there; the label under each frame describes the phase of the task. Bottom: the policy executed over the whole episode, one block per arbitration. The arbiter alternates between HIQL, which moves the free gripper to the next cube, and GCIQL, which grasps and places it.}
    \label{fig:episode}
\end{figure}

This selection problem is challenging for three reasons.
First, the policies cannot be ranked by their own value functions.
The most direct approach would pair each policy $\pi_i$ with its value function $V_i$ and execute $\pi_{i^\star}$ with $i^\star = \arg\max_i V_i(s, g)$, a policy-level form of generalized policy improvement \citep{barreto2017successor}.
This rule presumes that every $V_i$ measures the same quantity on the same scale, which does not hold across GCRL algorithms: $V_i$ may be a temporal-difference value, a quasimetric distance, or a contrastive score, and GCBC learns none.
Second, a policy should be judged by where its closed-loop behavior leads, yet at test time only the selected one is executed, so the consequences of the alternatives are never observed.
Third, frequent switching between policies can disrupt control.

Our key idea is to compare policies through the consequences they are predicted to produce. Our method, World-Model Policy Arbiter (\wmpa), rolls out each frozen policy in a learned world model while conditioning all policies on the same final goal. It scores the resulting imagined futures with a shared goal-conditioned value function, executes the highest-scoring policy for a fixed commitment interval, and repeats this procedure. Unlike action-level model-predictive control, \wmpa{} does not optimize new action sequences. Instead, it evaluates the behaviors already represented in the policy bank, which requires executing each policy just once at each imagined step, rather than a search over possible trajectories. All base policies remain frozen, and the selector requires no hand-designed progress signal, no rollouts in the real environment, and no privileged task knowledge.

Our paper makes the following contributions:
\begin{itemize}
    \item \textbf{Problem: Policy arbitration for offline GCRL.} We study test-time control over a bank of independently trained goal-conditioned policies as a policy-selection problem, which is distinct from selecting a single global winner or retargeting the subgoal of a single controller. We argue that offline GCRL is the natural setting for this problem: the policy bank incurs no additional data collection cost, because its members are trained entirely from the offline dataset without online interaction; moreover, goal-conditioning multiplies the situations in which different members excel.
    \item \textbf{Method: World Model Policy Arbitration.} \wmpa{} compares frozen policies by scoring their closed-loop imagined futures with a shared goal-conditioned value function, without policy retraining or privileged signals. 
    We also derive a conditional guarantee over the best fixed bank member: arbitration improves success when locally useful alternatives are available,
    true-dynamics scores separate inferior choices, and rollout-score errors rarely overturn that separation (\cref{app:wmpa_theory}).
    \item \textbf{Finding: Empirical gains from choosing frozen policies.} We evaluate \wmpa{} on \NumEnvs{} state-based OGBench datasets under the official protocol. We compare against every policy in the bank, a random-switching baseline with a matched commitment interval, and additional model-free and model-based baselines. We further analyze when arbitration is beneficial, using a per-episode hindsight oracle and policy-bank statistics.
\end{itemize}

\section{Background On Goal-Conditioned Reinforcement Learning}
\label{sec:prelim}
Following \citet{park2025ogbench}, we consider goal-conditioned reinforcement learning over a controlled Markov process
$\mathcal{M}=(\calS,\calA,\calG,\mu,p)$.
Here, $\calS$ is the state space, $\calA$ is the action space, $\calG$ is the goal space, $\mu\in\Delta(\calS)$ is the initial-state distribution, and $p(s'\mid s,a)\colon\calS\times\calA\to\Delta(\calS)$ denotes the transition dynamics, where $\Delta(\cdot)$ denotes the space of probability distributions over its argument. A goal-conditioned policy $\pi(a\mid s,g)\colon\calS\times\calG\to\Delta(\calA)$ maps the current state and goal to a distribution over actions.

Let $\nu\in\Delta(\calS\times\calG)$ denote the evaluation distribution over initial-state and goal pairs. Given $(s_0,g)\sim\nu$, a policy induces a trajectory $\tau=(s_0,a_0,s_1,\ldots,a_{T-1},s_T)$ according to the transition dynamics. Its performance is measured by expected goal-reaching success:
\[
    J(\pi)
    =
    \mathbb{E}_{(s_0,g)\sim\nu,\,
    \tau\sim p^\pi(\tau\mid s_0,g)}
    \left[\operatorname{Succ}(\tau,g)\right],
\]
where $\operatorname{Succ}(\tau,g)\in\{0,1\}$ indicates whether $\tau$
reaches $g$. This success criterion is used only for evaluation and is not available during training. Instead, goal-reaching supervision is constructed through hindsight goal relabeling, using states from logged trajectories as goals.

We assume access to a logged trajectory dataset $\calD=\{\tau^{(n)}\}_{n=1}^N$, where each trajectory $\tau^{(n)}=(s_0^{(n)},a_0^{(n)},s_1^{(n)},\ldots,s_{T_n}^{(n)})$ is collected as part of an offline dataset. We use $\calD$ to train the auxiliary world model and value function. In our experiments, $\calD$ is the OGBench training split, and all policies are trained offline using the same split. The general formulation, however, does not require the policies to be trained offline or from a shared dataset.

Rather than learning a single goal-conditioned policy, we are given a portfolio of $M$ frozen policies,
\[
    \PiBank=\{\pi_i(a\mid s,g)\}_{i=1}^M,
\]
whose members may differ in the GCRL algorithm or random seed used for training. All members share a common state, action, and goal interface, and the arbiter requires only black-box access to their action distributions. Internal value estimates and representations are neither assumed to be accessible nor required to be comparable across policies.

At arbitration time $t$, the arbiter selects a policy index $z_t\in\{1,\ldots,M\}$. The selected policy then controls the agent for a commitment interval of length $c$:
\[
    a_{t+j}\sim\pi_{z_t}(\cdot\mid s_{t+j},g),
    \qquad j=0,\ldots,c-1.
\]
The resulting controller is therefore a switching policy over the frozen portfolio members. The policy-portfolio objective is to maximize the expected goal-reaching success of this induced policy without modifying any member of $\PiBank$. The central problem is thus to determine, from the current state and goal, which portfolio member is expected to produce the most favorable future trajectory. 

\begin{figure}[t]
    \centering
    \includegraphics[width=0.95\linewidth]{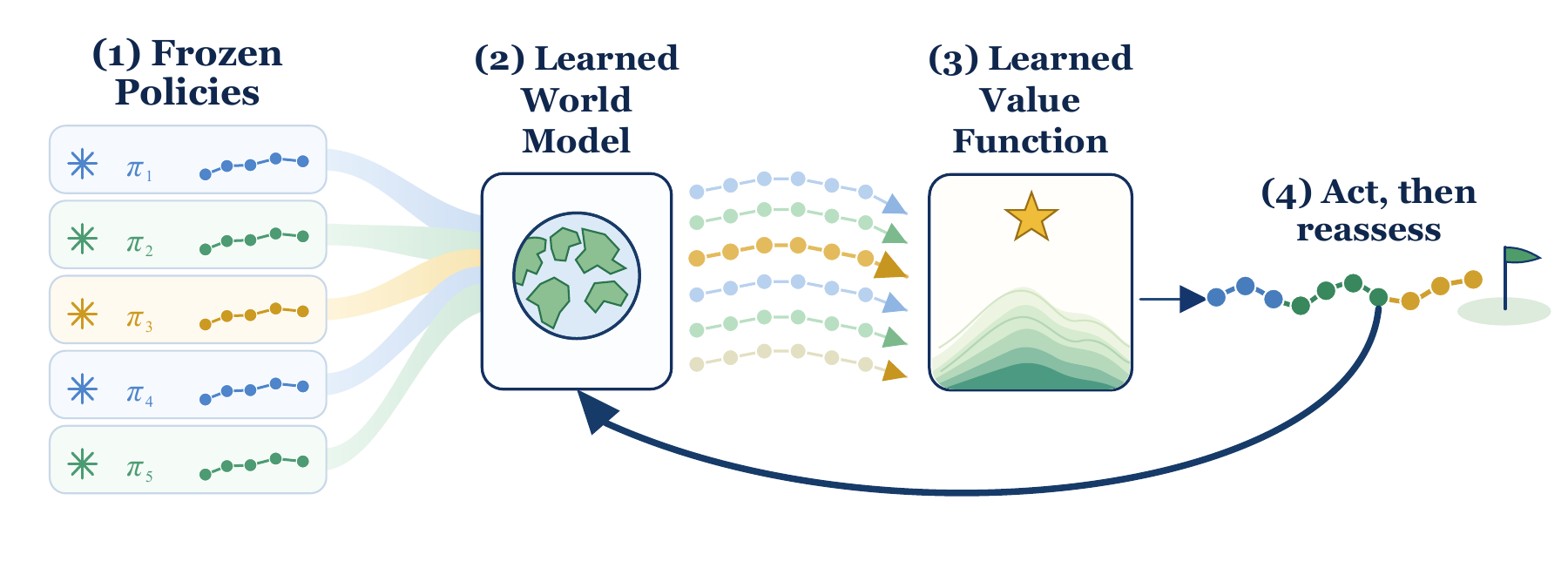}
    \vspace{-4mm}
    \caption{
    \textbf{WMPA overview.}
    (1) We begin with a frozen portfolio of heterogeneous GCRL policies.
    (2) \wmpa{} imagines the future induced by each policy through a shared
    learned world model.
    (3) A shared goal-conditioned value function scores the predicted futures
    on a common scale.
    (4) The highest-ranked policy is executed for $c$ steps, after which the process repeats.
    }
    \label{fig:overview}
\end{figure}

\section{Method: World-Model Policy Arbiter}
\label{sec:method}

At each arbitration point, \wmpa{} imagines $k$ steps of every frozen policy in a learned world model (\cref{sec:dynamics}), scores the imagined states with a single goal-conditioned value function (\cref{sec:value}), and executes the winner for $c=k$ environment steps (\cref{sec:counterfactual}); \cref{fig:overview} illustrates the loop.
Each step answers one challenge from \cref{sec:intro}: imagination supplies the unobserved consequences of the policies that are not selected, one shared value function puts all candidates on a common scale, and both commitment and value-based selection limit excessive switching between policies.

\subsection{Learned state-space world model}
\label{sec:dynamics}
For each environment, we train an ensemble of $E$ independently parameterized MLP dynamics models~\citep{kurutach2018model} directly in the environment state space. Throughout this section, all states are raw environment states; the fixed normalizer $\nu(s)=(s-\mu_s)/\sigma_s$, whose statistics are computed from the logged trajectories, is applied only at the network input. Each member predicts a normalized one-step delta and returns the next state,
\begin{align}
\hat{\delta}^{(e)}_t &= f_{\theta_e}\big(\nu(s_t),a_t\big),\\
\hat{s}^{(e)}_{t+1} &= \hat f_{\theta_e}(s_t,a_t) := s_t + \mu_{\Delta}+\sigma_{\Delta}\hat{\delta}^{(e)}_t,
\end{align}
where $\mu_{\Delta}$ and $\sigma_{\Delta}$ are fixed statistics computed from valid logged transitions. Multi-step predictions are obtained by applying the same $\hat f_{\theta_e}$ to its own outputs, with re-normalization at every input. The outputs are trained to match non-normalized environment states and are consumed by the frozen policies exactly like real observations. The members are trained on $H$-step windows by unrolling their own predictions and penalizing the normalized error at every step,
\begin{equation}
\mathcal{L}_{\mathrm{dyn}} = \frac{1}{E}\sum_{e=1}^{E}
\E_{\tau\sim\calD}
\left[\sum_{h=1}^{H} w_h\left\|\nu\big(\hat{s}^{(e)}_{t+h}\big)-\nu\big(s_{t+h}\big)\right\|_2^2\right],
\end{equation}
with a curriculum that increases the effective horizon from $1$ to $H$ over the first $10^5$ updates ($w_h$ uniform over the active steps). Predicting in state space isolates arbitration from representation learning and guarantees that imagined states are valid inputs to every frozen policy. Notably, the model is decoupled from the policies: learning it can be done independently of making changes to the policy bank.

\subsection{Goal-conditioned value function}
\label{sec:value}

The arbiter scores every candidate's imagined states with a single goal-conditioned value function $V(s,g)$, so policies trained with different objectives are compared on a common scale. Only the ordering matters: futures reached by different candidates from the same decision state should be ranked in the same order as their realized outcomes. Nothing in the arbiter depends on how $V$ is parameterized, and we use two parameterizations whose inductive biases match different kinds of task progress; we choose between them per task family (\cref{sec:setup}, \cref{app:hyperparams}). We also swap the two value heads to show the performance difference and discuss value-head choice in \cref{app:ablations}.

\paragraph{Metric value (maze, cube, scene).} Where progress is geometric, like moving through a maze, displacing objects toward target poses, we use a latent distance with separate state and goal encoders for each ensemble member,
\begin{equation}
V_{\phi_e}(s,g)
= -\left\|\varphi^{(e)}_{S}(s)-\varphi^{(e)}_{G}(\psi^{(e)}(g))\right\|_2,
\label{eq:metric-value}
\end{equation}
where $\psi^{(e)}$ is a low-dimensional goal representation, following \citet{kang2026lavl}. The distance orders imagined state--goal pairs even where closed-loop rollouts leave the support of the logged data, and the distinct encoders keep reachability directional ($V_{\phi_e}(s,g)\neq V_{\phi_e}(g,s)$ in general). Each value member is trained jointly with its corresponding dynamics member by an expectile temporal-difference objective on relabeled goals with a sparse reward ($0$ at attainment and $-1$ otherwise),
\begin{equation}
\mathcal{L}_{\mathrm{value}}
= \frac{1}{E}\sum_{e=1}^{E}
\E\!\left[
\left|\kappa-\mathbf{1}\{\bar\delta_e<0\}\right|
\big(y_e-V_{\phi_e}(s,g)\big)^2
\right],
\quad
y_e = r(s,g)+\gamma\,m(s,g)\,
\bar V_{\bar\phi_e}(s',g),
\label{eq:expectile}
\end{equation}
where $\bar V$ is a target network, $m(s,g)=0$ at attainment, $\bar\delta_e=y_e-\bar V_{\bar\phi_e}(s,g)$, and the expectile satisfies $\kappa>0.5$ (settings in \cref{app:wm}).

\paragraph{Direct value (puzzle).} On the puzzle datasets progress is combinatorial: a state is a pattern of buttons, a press toggles a button together with its neighbours, and the number of presses to the goal pattern is not monotone in embedding distance. Thus, we use a value with no geometric structure,
\begin{equation}
V_{\theta}(s,g) = \mathrm{MLP}_{\theta}\big([\,s;\,g\,]\big),
\label{eq:direct-value}
\end{equation}
trained by implicit Q-learning \citep{kostrikov2022iql} on the same dataset. Two action-value heads $Q_{\omega,i}(s,a,g)$ are regressed toward the one-step target $r(s,g)+\gamma\,m(s,g)\,V_{\theta}(s',g)$, and the state value is the $\kappa$-expectile of the target value function $\bar Q=\min_i \bar Q_{\omega,i}$ under the dataset's actions,
\begin{equation}
\mathcal{L}_{V}
= \E\Big[\left|\kappa-\mathbf{1}\{\bar Q(s,a,g)-V_{\theta}(s,g)<0\}\right|\big(\bar Q(s,a,g)-V_{\theta}(s,g)\big)^2\Big],
\label{eq:direct-expectile}
\end{equation}
so that $V_{\theta}$ approximates the value of the best dataset-supported action at each relabeled state--goal pair. It is the state-value network of the bank's GCIQL member, trained with the reference implementation's settings and reused unchanged as the scorer. Being one function applied to every candidate's imagined states, it keeps scores comparable across the bank exactly as the metric value does.

\subsection{Model-Based Policy Rollout and ranking}
\label{sec:counterfactual}
At each arbitration point, every frozen policy is rolled out in closed loop for $k$ model steps, acting deterministically on its own imagined states. All policies in the bank are Gaussian actors whose deterministic mode is the behavior deployed and evaluated by the benchmark, so the imagined branch reproduces the behavior that would be executed; we do not evaluate stochastic rollouts obtained by sampling actions from the policy distributions. For policy $i$ and ensemble member $e$,
\begin{align}
\hat{s}^{i,e}_{0} &= s_t,\\
\hat{a}^{i,e}_{h} &= \pi_i(\hat{s}^{i,e}_{h},g),\\
\hat{s}^{i,e}_{h+1} &= \hat{f}_{\theta_e}(\hat{s}^{i,e}_{h},\hat{a}^{i,e}_{h}),
\qquad h=0,\ldots,k-1,
\end{align}
where $\hat f_{\theta_e}$ is the $e$-th dynamics-model ensemble member of \cref{sec:dynamics}.
We score each candidate by the highest value imagined state along the rollout,
\begin{equation}
S_i(s_t,g)=\max_{1\le h\le k}
\frac{1}{E}\sum_{e=1}^{E}
V_{\phi_e}(\hat{s}^{i,e}_{h},g).
\label{eq:score}
\end{equation}
The max operator allows a policy to receive credit for reaching the goal before the end of the planning horizon. \cref{tab:ablations} ablates both reductions (mean and final-state alternatives over the horizon; min and mean$-$std over ensemble members). \wmpa{} selects the policy index with the highest score $i^\star=\arg\max_{i} S_i(s_t,g)$ and executes $\pi_{i^*}$ in the real environment for $c$ steps, and then repeats the arbitration (\cref{alg:wmpp}, \cref{app:hyperparams}).

\paragraph{Rollout length and commitment.}
Each arbitration costs $MEk$ dynamics and $MEk$ value evaluations ($M{=}6$ policies, $E{=}3$ ensemble members), or $MEk/c$ per environment step.
Setting $k{=}c$ fixes this at $ME$, which adds about $3$--$4$\,ms per step on 8 CPU threads (\cref{app:runtime}); the horizon is thus a single hyperparameter, set per task family (\cref{tab:settings}), trading off reactivity against foresight at constant amortized cost.
The large gains are robust to $k$ and $c$, even when $k \neq c$ (\cref{tab:kc-grid}).

\section{Experiments}
\label{sec:setup}
We evaluate \wmpa{} on \NumEnvs{} state-based OGBench datasets \citep{park2025ogbench}: two maze-navigation datasets (\texttt{pointmaze-medium} and \texttt{antmaze-large}) as well as six cube (\texttt{single} to \texttt{triple}), two scene, and eight puzzle (\texttt{3x3} to \texttt{4x6}) manipulation datasets, which together cover the \texttt{play} and \texttt{noisy} data-collection variants (\cref{app:hyperparams}). OGBench exposes the same observation, action, goal, and evaluation interfaces to policies trained with different objectives, and the use of state observations isolates arbitration from visual representation learning. The portfolio contains the six OGBench reference learners (GCBC, GCIVL, GCIQL, QRL, CRL, and HIQL), each trained with the unmodified reference implementation and per-family hyperparameters for $10^6$ steps. The policy bank holds one checkpoint per algorithm, and the study is replicated over three bank seeds. All policies remain frozen; the scores of the individual members are reported in \cref{tab:main}. All methods are evaluated on the same fixed episode seeds, so every comparison is paired per episode; reported intervals are 95\% hierarchical bootstraps over bank seeds and then episodes, and a gain is called significant when the interval of the paired difference excludes zero (\cref{app:bank-stats}).

\subsection{Policy Arbiter improves goal reaching}
\label{sec:main_res}
\Cref{tab:main} reports the success rate of every policy in the bank and of \wmpa{}.
Across the \NumEnvs{} datasets, \wmpa{} achieves a mean success rate of \MeanWMPP\%, compared with \MeanBest\% for the best policy and \MeanRandom\% for random switching with the same interval and policy bank (\cref{sec:ablation}).
\wmpa{} attains the highest mean on \NumBetterMean{} datasets, and its gain over the best policy is significantly positive on \NumSigUp{} datasets, not significant on \NumNS{}, and significantly negative on \NumSigDown{}.
The largest improvements occur on multi-object manipulation tasks, including \GainSceneNoisy{} percentage points on \texttt{scene-noisy} and \GainCubeDoubleNoisy{} on \texttt{cube-double-noisy}; the per-dataset confidence intervals are reported in \cref{app:bank-stats}.

\begin{table}[t]
\caption{\textbf{Success rate (\%) on OGBench under the official protocol} ($5$ goals $\times$ $50$ episodes, $3$ bank seeds). We report every bank policy, the best policy for each dataset which is the bank policy with the highest three-seed mean success on the reported episodes, and \wmpa{}, which arbitrates over the same frozen policies with the $(k,c)$ setting of its task family (\cref{app:hyperparams}). Each entry is the mean success rate $\pm$ the half-width of its $95\%$ hierarchical-bootstrap interval (resampling bank seeds, then episodes), rounded to the nearest integer. Bold indicates values within $95\%$ of the row maximum.
}
\label{tab:main}
\centering
\scriptsize
\setlength{\tabcolsep}{3pt}
\resizebox{\linewidth}{!}{\input{tables/main_results}}
\end{table}

We compare \wmpa{} to action-level model-predictive control (MPC) applied to the best policy with the same $k$ and value function, re-planned at every environment step, in \Cref{tab:search}.
MPC, given $k$ times the model budget of \wmpa{} ($18k$ versus $18$ imagined steps per environment step), averages \AblMeanPolicyMpc\% against \MeanWMPP\% for \wmpa{}; it is significantly worse than arbitration on \AblNumDownPolicyMpc{} of \AblNumEnvsPolicyMpc{} datasets and significantly better on \AblNumUpPolicyMpc{}, with its largest deficits on the datasets where arbitration gains most (\texttt{cube-single-play}, \texttt{cube-double-play}, \texttt{scene-play}), which supports the benefit of arbitrating among different policies beyond locally refining the actions of a single best policy.

\subsection{Does the world model improve performance? Usually.}
\label{sec:ablation}

We study a \randomswitch{} ablation which removes the entire learned model: it uses neither the world-model rollouts nor the value head, and draws a policy uniformly from the bank at every arbitration point. Everything else is shared with \wmpa{}, so the comparison isolates the learned scoring of imagined futures. Under per-episode pairing, \wmpa{} outperforms \randomswitch{} by \MeanWMPPminusRandom{} points on average, with significant gains on \NumWMPPvsRandomSig{} datasets. Random switching alone is significantly below the best policy on \NumRandomSigDown{} datasets and significantly above it on \NumRandomSigUp{}. \wmpa{}'s major gains therefore come neither from the commitment schedule nor from cycling through policies, but from model-based candidate ranking. Full results of this ablation are presented in \cref{tab:random-c}.

To check whether the main gains of policy switching come from avoiding states that the current policy cannot escape, we design an event-triggered model-free baseline. It retains the selected policy and switches to a randomly chosen other policy only when the agent stalls. Stalling is measured as the net displacement of the full observation vector over the last $m$ steps, $\|\nu(s_t)-\nu(s_{t-m})\|_2$, in the standardized space of \cref{sec:dynamics}, so the threshold $\varepsilon=0.05$ is dimensionless and is fixed across all datasets without tuning. \Cref{app:random-c} evaluates this rule for $m\in\{5,10,25,50\}$: its success varies little with $m$, and it does not recover the manipulation gains, so a displacement trigger is a poor proxy for the situations in which switching is beneficial (\cref{tab:stall}).

We also isolate the effect of dynamics-model accuracy on the performance of the method. To this end, \cref{tab:sim-all} replaces the learned dynamics with true-simulator branches on every dataset. This substitution raises the success rate by \AblMeanDeltaSimSel{} points on average, with significant gains on \AblNumUpSimSel{} of \AblNumEnvsSimSel{} datasets, and never significantly lowers it, so model error costs some ranking quality but is not the main factor that limits the method.

A Q-select ablation tests whether world-model rollouts are necessary at all: on the same bank and episodes (\cref{tab:scorer-selector}), Q-select ranks the action proposed by each member at the current state with the twin-Q value function of the GCIQL member, without imagined states. Re-selecting at every step ($c{=}1$), Q-select is a strong selector: it significantly exceeds the best policy on \MechNumAboveBestQselOne{} of \MechNumEnvsQselOne{} datasets, and on the puzzle family, where \wmpa{} scores with the $V$ of the same member, its gap to \wmpa{} is \MechFamFullQselOnePuzzle{} points on average. This is expected, since $Q(s,a,g)=r+\gamma V(s',g)$ already summarizes the future. Rollouts nevertheless matter for two reasons. First, Q-select requires an action-value function, whereas the metric value of \cref{sec:value} can be used only through imagined states and is the better scorer on the maze, cube, and scene families; averaged over all datasets, \wmpa{} exceeds Q-select at $c{=}1$ by \MechMeanDeltaFullQselOne{} points. Second, scoring only the first action says nothing about the $c{-}1$ steps that follow, so at the matched interval $c{=}k$, Q-select degrades while the $k$-step rollout does not: direct-value \wmpa{} exceeds Q-select by \MechMeanDeltaDirectQselK{} points on average, with significant gains on \MechNumUpDirectQselK{} datasets and no significant losses. Rollouts thus enable the better value head and make arbitration safe to amortize over $c$ steps, which keeps the budget at $ME$ model calls per step.

Scoring each rollout with the policy's own value function instead fails even in the setting most favorable to it: restricted to GCIQL, GCIVL, and HIQL, whose value functions estimate the same quantity, the selection collapses onto the most optimistic value function and is significantly worse than shared scoring on 14 of 18 datasets (\cref{tab:native-critic}).

\subsection{Arbitration benefits from complementary behaviors within episodes} 
\label{sec:when}

\wmpa{} can only select behaviors that already exist in the bank.
Here, we investigate whether the gains of arbitration arise solely from selecting the best policy for each episode or whether they also rely on within-episode composition.
We construct the per-episode oracle by evaluating every policy in the bank on each episode under a fixed seed and defining the oracle performance on that episode as the maximum performance across all policies.

\begin{figure}[htbp]
    \centering
    \includegraphics[width=1\linewidth]{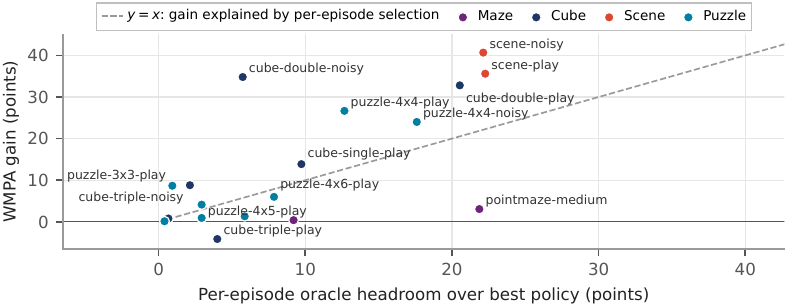}
    \caption{\textbf{What makes arbitration work?} Each point is a dataset. Horizontal axis: success added over the best policy by the per-episode oracle, which picks a successful bank policy per episode but never switches within it, so it bounds per-episode selection, not \wmpa{}. Vertical axis: success added by \wmpa{}. Above the dashed $y=x$ line, the extra success must come from switching within an episode; \wmpa{} is significantly above the oracle on \NumAboveUnionSig{} datasets and below it on \NumBelowUnionSig{} (\cref{app:bank-stats}).}

    \label{fig:when}
\end{figure}

In \cref{fig:when}, we plot the gain of \wmpa{} against the gain of the per-episode oracle, where both gains are measured relative to the best policy in the bank.
First, \wmpa{} significantly outperforms the per-episode oracle on \NumAboveUnionSig{} datasets, with large margins on \texttt{cube-double-play} (\WmppCubeDoublePlay\% versus \UnionCubeDoublePlay\%) and \texttt{scene-play} (\WmppScenePlay\% versus \UnionScenePlay\%).
Because the oracle already selects a successful policy for each evaluation episode whenever one exists, per-episode selection alone cannot account for these gains; instead, they indicate a benefit of composing policies \emph{within} an episode (\cref{fig:episode} shows one such episode).
Second, neither a large gap between the best and the second-best policy nor limited headroom between the best policy and the oracle precludes gains from switching.
On \texttt{puzzle-3x3-play}, \wmpa{} improves over the best policy by \GainPuzzleThreebyThreePlay{} percentage points, even though the success rate of the second-best policy is \DgapPuzzleThreebyThreePlay{} percentage points lower than that of the best policy and the headroom to the per-episode oracle is only \HeadroomPuzzleThreebyThreePlay{}\%.
Similarly, on \texttt{cube-single-play}, the second-best policy trails the best by \DgapCubeSinglePlay{} percentage points and the oracle provides only \HeadroomCubeSinglePlay{} points of headroom, yet the gain is \GainCubeSinglePlay{} percentage points.
Both gains exceed the corresponding static headroom, which shows that within-episode composition can improve performance even when a single member is substantially stronger than the others overall.
Moreover, the bank need not be drawn from a single dataset: pooling the policies trained on the \texttt{play} and \texttt{noisy} data of an environment into one twelve-member bank never degrades the performance of \wmpa{} and improves it whenever the added members contribute new competencies (\cref{app:bank-composition}). A complementary analysis in \cref{app:wmpa_theory} expresses the gain over the best fixed bank member as accumulated one-chunk advantages and bounds the loss from selecting inferior members using score separation and rollout-score error.

\section{Related Work}
\label{sec:related}
\paragraph{Goal-conditioned reinforcement learning.}
GCRL learns policies conditioned on a desired goal, using supervised, value-based, contrastive, quasimetric, and hierarchical objectives \citep{eysenbach2022contrastive,wang2023quasimetric,park2023hiql}; OGBench provides common datasets and reference implementations, and it shows that the strengths of these methods vary across navigation, stitching, stochasticity, and manipulation \citep{park2025ogbench}. Model-based GCRL uses a model to improve the policy being trained: planning with a learned model to fine-tune offline policies \citep[GOPlan;][]{wang2024goplan}, planning exploration goals \citep[PEG;][]{hu2023peg}, modeling transitions between arbitrary buffer states for goal navigation \citep{duan2024mun}, synthesizing goal-reaching data backward \citep{hoftmann2023backward}, or constructing self-curricula \citep{gong2024gcpo}; see \citet{liu2022gcrl} for a survey. Stronger recent learners \citep{dong2026hiflowq} and abstraction-based state--goal representations \citep{wibault2026abstraction} are 
complementary to \wmpa{}: they could serve as additional bank members and can supply the value head. Goal-conditioned hypernetworks even generate policies on demand \citep{faccio2023goalcond}. All of these approaches modify or create the policy; we leave every policy frozen and use the model only to choose among them.

\paragraph{Test-time planning with frozen policies and world models.}
Existing test-time methods devote computation to \emph{producing} behavior. Goal-conditioned policies serve as low-level controllers for planning through intermediate goals \citep{nasiriany2019planning}; Test-Time Graph Search builds a graph over dataset states and feeds subgoals to a single frozen policy \citep{opryshko2026ttgs}; CompPlan composes subgoal-conditioned instances of a single policy family through policy-conditioned jumpy models \citep{farebrother2026compplan}; and ALPS plans over Laplacian subgoal clusters with a one-step model \citep{shehmar2026laplacian}. All these methods ask \emph{which subgoal should one controller pursue?}, whereas \wmpa{} keeps the final goal fixed and asks \emph{which controller should act now?} CompPlan is closest in spirit, since it also sequences pre-trained policies with a learned model; but its policies form one family parameterized by a subgoal or latent $z$, its jumpy models condition on $z$, and it searches over which $z$ to pursue next, whereas \wmpa{} uses a policy-agnostic dynamics model to arbitrate among heterogeneous policies trained by different algorithms, never changing the goal.
 Hand-offs between a task policy and a recovery policy in safety-critical offline GCRL \citep{cao2024rbsl} are policy switching with two controllers and a hand-set cost threshold as the trigger; \wmpa{} switches over an arbitrary bank, with the model-based score as the trigger. World models, in turn, are used for action-level planning or training an amortized policy \citep{schrittwieser2020muzero,hafner2023dreamerv3}; sampling-based MPC with learned models \citep{williams2017mppi,hansen2022tdmpc} motivates the action-level baseline of \cref{app:ablations} which perturbs the first action, and hierarchical or reset-free world models \citep{schiewer2024hierarchical,yang2024resetfree} likewise generate behavior rather than evaluate it. \wmpa{} uses the model only to evaluate closed-loop behaviors already in the bank.

\paragraph{Policy switching, options, and learned hierarchies.}
Selecting among fixed policies by simulating each is a long-standing idea: policy switching and parallel rollout execute, at every state, the base policy with the highest simulated return, and both provably improve upon every base policy \citep{chang2004parallel,bertsekas1999rollout}; generalized policy improvement (GPI) achieves the same result with per-policy action-values, provided they are comparable \citep{barreto2017successor}. In the options framework, each frozen member of $\mathbf{\Pi}$ is an option with a fixed-duration termination condition, and re-arbitrating every $c$ steps is a scheduled form of option interruption, which is guaranteed not to degrade performance when the switching value is exact \citep{sutton1999between}. \wmpa{}
is a descendant of this lineage with a learned model, a truncated horizon, and goal conditioning: the model supplies the rollouts that the classical methods obtain from a simulator, and the shared value function avoids the need for the comparable per-policy action-value estimates required by GPI.
Model-based hierarchical RL is a distinct lineage. Director selects latent goals for a worker trained in imagination \citep{hafner2022director}; SkiMo plans with a skill-dynamics model learned jointly with a skill VAE \citep{shi2022skimo}; SPlaTES shapes skills to be predictable by an abstract world model and plans over them with iCEM \citep{gurtler2025splates}; Puppeteer trains a separate high-level world model for each task on top of a frozen tracking controller \citep{hansen2025puppeteer}; and a learned gate over a frozen planner decides how much search to perform \citep{muppidi2026finding}. All of these methods learn the high level, most also co-adapt the low level, and their models generate or search over the executed behavior. \wmpa{} learns neither level: the candidates are complete goal-conditioned policies trained independently with different algorithms, the dynamics model is trained once on the dataset without reference to any policy, and the model only arbitrates among them.

\section{Limitations and Future Work}
\label{sec:limitations}
This study uses state observations and per-environment transition models. Pixel observations would require either a latent world model with a shared visual encoder across all policies or a world model that produces realistic images. We leave visual environments to future work. Additionally, \wmpa{} selects among existing behaviors and cannot create new ones: when every member is at the task floor, one member is already saturated, or the members fail on the same subtasks, arbitration typically does no better than the best policy.
Finally, arbitration adds a few milliseconds of computation per step (\cref{app:runtime}), a modest overhead; future work could remove it by using the world model offline, either to distill the arbiter's choices into a single policy or to generate corrected training data in the style of DAgger.

\section{Conclusion}
\label{sec:conclusion}

We introduced World-Model Policy Arbiter (\wmpa{}), which selects among frozen goal-conditioned policies at test time by rolling each forward in a learned world model, scoring the imagined futures with a shared value function, and executing the highest-scoring policy for a short commitment interval before selecting again. Without retraining any policy, it improves on the best policy in its bank by \MeanGain{} points on average and by 33--41 points on several multi-object manipulation tasks. Notably, it significantly exceeds a per-episode hindsight oracle on \NumAboveUnionSig{} datasets, indicating a benefit of composing policies within an episode. Taken together, these results suggest viewing world models as test-time evaluators of existing behaviors, not only as generators of new actions.

\label{page:endconclusion}
\section*{Reproducibility Statement}
Our code is publicly available at \url{https://github.com/junwei0102/World-Model-Policy-Arbiter}. It includes the training entry points for the policy bank, which wrap the unmodified OGBench reference implementations, and for the world model with its value head; the evaluation script that runs \wmpa{} and every control and ablation reported in this paper; and the scripts that aggregate the per-episode logs into success rates and paired hierarchical-bootstrap intervals. A per-dataset registry records the bank, the checkpoints, and the reported interval for all \NumEnvs{} datasets, and the software versions are listed. Reset seeds are a deterministic function of the dataset, the goal, and the episode index, so any method can be re-evaluated on exactly the initial states used here. All datasets are the public OGBench releases; trained checkpoints and raw evaluation logs are omitted because of their size. 



\label{page:endmain}

\bibliographystyle{iclr2027_conference}
\bibliography{references}
\appendix

\section{Per-Dataset Planning Settings}
\label[appendix]{app:hyperparams}

\Cref{alg:wmpp} summarizes the arbitration loop of \cref{sec:method}.

\begin{algorithm}[htbp]
\caption{World-Model Policy Arbiter}
\label{alg:wmpp}
\begin{algorithmic}[1]
\Require frozen policies $\PiBank$, dynamics ensemble $\{\hat f_{\theta_e}\}_{e=1}^E$, value ensemble $\{V_{\phi_e}\}_{e=1}^E$, goal $g$, planning horizon $k$, commitment $c$
\While{episode not terminated}
    \For{$i=1,\ldots,M$}
        \State Roll out $\pi_i$ for $k$ imagined steps through every ensemble member
        \State Compute $S_i$ with \cref{eq:score}
    \EndFor
    \State $i^*\gets\arg\max_i S_i$
    \State Execute frozen policy $\pi_{i^*}$ in the real environment for up to $c$ steps
\EndWhile
\end{algorithmic}
\end{algorithm}

\Cref{tab:settings} lists the settings used throughout \wmpa{} main results (\cref{tab:main}).

\begin{table}[htbp]
\caption{\textbf{Settings per task family}. The value head is the parameterization of the scorer applied to every candidate's imagined states (metric: latent-distance value with separate encoders; direct: a single network on the state--goal pair, \cref{sec:value}); the single interval $k{=}c$ is the imagination horizon and commitment of \wmpa{} and, identically, the re-draw interval of \randomswitch{}; the remaining columns ($H$, $\gamma$, $\kappa$, $\lambda_{\mathrm{smooth}}$) are the training settings of the world model and its value head.}
\label{tab:settings}
\centering
\small
\setlength{\tabcolsep}{5pt}
\resizebox{\linewidth}{!}{\input{tables/settings_family}}
\end{table}

\begin{table}[t]
\caption{\textbf{Sensitivity to the imagination horizon $k$ and the commitment $c$} on three datasets with large gains ($k \ge c$). A row fixes $k$, a column fixes $c$, and the diagonal is $k=c$; the cell reported in \cref{tab:main} is underlined. Success is the mean $\pm$ the half-width of the $95\%$ interval, switches are policy changes per episode, and latency is the mean wall-clock time per environment step on the CPU evaluation nodes. A cell costs $18k/c$ model calls per step, so only the diagonal keeps the budget of \cref{tab:main}.}
\label{tab:kc-grid}
\centering\small
\resizebox{\linewidth}{!}{\input{tables/kc_grid}}
\end{table}

\Cref{tab:kc-grid} varies $k$ and $c$ on three datasets with large gains. The interval is not tuned extensively for each dataset; instead, it is kept the same within each dataset family. Across all \KcGridNumCells{} cells, every cell retains a gain of at least \KcGridMinGain{} points over the best policy. This table shows that the large gains on these manipulation tasks are robust to the hyperparameter setting and are attributable to the \wmpa{} algorithm itself.

\section{No-Model Switching Controls}
\label[appendix]{app:random-c}

\Cref{tab:random-c} reports \randomswitch{} alongside \wmpa{} on every dataset, with both methods using the same commitment interval $c$, so the two columns differ only in how the next policy is chosen: a uniform draw from the bank, or the learned score. Uninformed re-drawing on the same schedule averages \MeanRandom\% against \MeanWMPP\% for \wmpa{}; paired per episode, \wmpa{} is significantly better on \NumWMPPvsRandomSig{} datasets (\cref{sec:ablation}).

\begin{table}[htbp]
\caption{\textbf{\randomswitch{} versus \wmpa{} at a matched commitment interval} (the per-family value of $c$ in \cref{tab:settings}). Best is the best bank policy on each dataset (\cref{tab:main}). In each row, entries that reach at least $95\%$ of the maximum over the three columns are in bold.}
\label{tab:random-c}
\centering
\scriptsize
\setlength{\tabcolsep}{3pt}
\input{tables/random_by_c}
\end{table}

\Cref{tab:stall} reports results for the stall-restart controller (\cref{sec:ablation}), which switches in response to \emph{evidence of being stuck} rather than on a fixed schedule.
The controller draws its first policy uniformly at random from the bank and keeps it until the displacement over the last $m$ steps falls below $\varepsilon$; it then restarts with one of the \emph{other} members, drawn uniformly at random, and waits at least $m$ steps before the trigger can fire again.
Because the controller uses neither a world model nor a value function, it serves as a test of whether the gains from switching come from escaping stalls rather than from choosing the right policy.
Across $m\in\{5,10,25,50\}$ (shorter windows trigger restarts more readily), the controller exceeds \wmpa{} on only \NumStallBeatsWmpp{} of the \NumEnvsStall{} datasets: \texttt{cube-triple-play}, where all methods have near-zero success, and \texttt{pointmaze-medium-navigate}, where it reaches 94--99\% success, compared with 79\% for \wmpa{}.
\randomswitch{} also reaches this range there (\cref{tab:random-c}).
A deterministic policy that stalls keeps stalling, and a restart with a different member can escape the stall, whereas \wmpa{} may keep re-selecting the stalled policy in such states.
A stall trigger is therefore a natural addition to \wmpa{} on tasks where policies stall.
On the manipulation datasets, however, the stall-restart controller does not recover the gains of \wmpa{}, so these gains do not come from escaping stalls alone.

\begin{table}[htbp]
\caption{\textbf{Stall-restart control}: a no-model controller that keeps the incumbent policy and switches to a random \emph{other} policy only when the normalized displacement over the last $m$ steps falls below $\varepsilon$; same bank, episodes, and protocol as \cref{tab:main}. Numbers within $95\%$ of the row maximum are in bold.}
\label{tab:stall}
\centering
\small
\setlength{\tabcolsep}{4pt}
\input{tables/stall}
\end{table}

\section{Model-Based Controls and Scorer Comparisons}
\label[appendix]{app:ablations}

\Cref{tab:search} is discussed in \cref{sec:main_res} and \cref{tab:sim-all} is discussed in \cref{sec:ablation}. \Cref{tab:search} exchanges the portfolio for every-step action-level search on the best policy with the same horizon $k$ and value; \cref{tab:sim-all} replaces the learned dynamics with true-simulator branches on every dataset.

\Cref{tab:scorer-selector} varies only the value head or the selector, holding everything else fixed, on all eighteen datasets. Q-select, the model-free selector built from the twin-Q value function, is significantly above the best policy on \MechNumAboveBestQselOne{} of \MechNumEnvsQselOne{} datasets at $c{=}1$ and on \MechNumAboveBestQselK{} at $c{=}k$; compared with direct-value \wmpa{}, which scores imagined states with the state-value function $V$ of the same agent, it trails by \MechMeanDeltaDirectQselK{} points on average at a matched commitment interval, significantly so on \MechNumUpDirectQselK{} of \MechNumEnvsDirectQselK{} datasets. Replacing the default value head with the alternative reduces mean success on most datasets and lowers macro-average success by approximately 20 percentage points. The effect is nevertheless task dependent: on cube-triple-play, the direct head raises mean success from 1\% to 9\%.

\Cref{tab:native-critic} tests whether the value functions of the policies themselves can replace the shared value function in the setting most favorable to them.
In this setting, the bank is restricted to GCIQL, GCIVL, and HIQL, whose value functions are trained to estimate the same quantity (all three learn IQL-style value functions with the same reward and discount factor).
Each value function scores the imagined rollout of its own policy (own-value scoring), while the world model and the evaluation episodes remain unchanged.
Own-value scoring retains none of the gain obtained with the shared value function.
Averaged over the \NumEnvs{} datasets, the success rate of own-value scoring is 4 percentage points below that of the best of the three members, whereas scoring with the shared value function improves upon this member by 11 percentage points.
The shared value function also yields significant gains over this member and over own-value scoring on 12 and 14 datasets, respectively.
Under own-value scoring, the selection collapses onto the member with the most optimistic value function: on most datasets, the same member is selected at the first arbitration point in more than 90\% of the episodes and is rarely replaced thereafter.
For instance, on \texttt{antmaze-large-navigate}, the arbiter persistently selects GCIVL, which succeeds in 19\% of the episodes on its own, rather than HIQL, which succeeds in 90\%.
These results indicate that scoring every rollout with a single value function is essential for a meaningful comparison among the policies.

\Cref{tab:value-head-selection} presents a post hoc diagnostic analysis of value-head ranking quality; this analysis did not influence the value-head choices reported in \cref{tab:main}. We assess each head by measuring how closely its ranking of states agrees with the observed time to reach a shared goal in offline trajectories. We draw a state $s_a$ and a later state $g$ from one trajectory, and a state $s_b$ from a second trajectory that also reaches $g$ under the success predicate of the environment; a head is counted as concordant when it assigns the higher value $V(\cdot,g)$ to the state with the shorter observed time to reach $g$. Pairs whose arrival times differ by fewer than $20$ steps are discarded, and confidence intervals are obtained by bootstrapping over trajectories. The concordance criterion selects the metric head on \texttt{cube-double-play} and \texttt{scene-play} and the direct head on \texttt{puzzle-4x4-play}, in agreement with the heads used in \cref{tab:main}; on \texttt{cube-triple-play}, it favors the direct head, which is indeed the better of the two in online evaluation (\cref{tab:scorer-selector}). Taken together, these results suggest that this criterion can guide the selection of the value head on new datasets for which a goal-attainment predicate is available.

\begin{table}[htbp]
\caption{\textbf{WM action search vs.\ arbitration} with the same horizon
$k$ and the same value function on every dataset. The action search applies
sampling-based MPC on top of the best policy. At every environment step, the
candidates are the deterministic action $\pi(s_t,g)$ of the policy and
$N{-}1{=}5$ Gaussian perturbations of this action ($\sigma{=}0.2$). Each
candidate is imagined for $k$ steps: only the first action differs across
candidates, and the remaining $k{-}1$ steps follow the deterministic output
of the policy in closed loop on the imagined states, exactly as in the
rollout of a candidate policy in \wmpa{}. The candidates are scored as in
\wmpa{}, the first action of the best candidate is executed, and the search
is repeated at the next step ($c{=}1$). The action search uses
$N E k = 18k$ model steps per environment step, compared with $M E = 18$ for
\wmpa{}. The two budgets are equal on the mazes, whereas the budget of the
action search is $5\times$ larger on the cube family and $10\times$ larger
on the scene and puzzle families. On the puzzle family, both methods use the
direct value. Selecting the best policy per dataset as a base for MPC serves as oracle. Numbers that are at least $95\%$ of the row maximum are shown
in bold.}

\label{tab:search}
\centering
\small
\setlength{\tabcolsep}{4pt}
\resizebox{\linewidth}{!}{\input{tables/search}}
\end{table}

\begin{table}[htbp]
\caption{\textbf{True-simulator branches} inside \wmpa{} at the selected cell on every dataset (same value head and commitment; simulator steps replace model calls). $\Delta$ is the paired contrast simulator $-$ \wmpa{}; on the puzzle family the branches are scored with the direct value at the reported cell. Numbers within $95\%$ of the row maximum are shown in bold (${}^{*}$: interval excludes zero).}
\label{tab:sim-all}
\centering
\small
\setlength{\tabcolsep}{4pt}
\resizebox{\linewidth}{!}{\input{tables/sim_all}}
\end{table}

\section{Score Ablations}
\label[appendix]{app:scoreabl}

\Cref{tab:ablations} perturbs the score of \cref{eq:score} at the reported cell on every dataset: horizon aggregation (max over $h\le k$ $\rightarrow$ final state / mean) and ensemble reduction (mean over members $\rightarrow$ $\min$ / mean$-$std). Across datasets the reductions move the macro-average by at most a point (final state \AblMeanDeltaAggLast, mean \AblMeanDeltaAggMean, $\min$ \AblMeanDeltaEnsMin, mean$-$std \AblMeanDeltaEnsLcb), so neither the optimism of the max nor the ensemble mean carries the gains; pessimistic reductions do not eliminate the low-success failure cases.

\begin{table}[t]
\caption{\textbf{Comparison of value heads and selectors} at a fixed cell on all the datasets. The world model, the bank, and the evaluation episodes are held fixed; only the scorer or the selector varies. \wmpa{} scores imagined states with either the metric value or the direct value; $^{\dagger}$ marks the head reported in \cref{tab:main}. Q-select uses the same twin-Q value function to rank the action proposed by each policy at the current state, $i^*=\arg\max_i\min_j Q_{\omega,j}(s,\pi_i(s,g),g)$, without any imagined states, re-selecting either at every step ($c{=}1$) or every $c{=}k$ steps. The last column is the paired contrast between the two selectors that share the same underlying agent, direct-value \wmpa{} minus Q-select at $c{=}k$, with its $95\%$ interval: it isolates the contribution of scoring imagined successor states with the value function. Entries are the mean success rate $\pm$ the half-width of the $95\%$ hierarchical-bootstrap interval, and values within $95\%$ of the row maximum are shown in bold (${}^{*}$: interval excludes zero).}
\label{tab:scorer-selector}
\centering
\small
\setlength{\tabcolsep}{4pt}
\IfFileExists{tables/scorer_selector.tex}{\resizebox{\linewidth}{!}{\input{tables/scorer_selector}}}{(pending)}
\end{table}

\begin{table}[t]
\caption{\textbf{Own-value function versus a shared value function} on the three bank members whose value networks estimate the same quantity (GCIQL, GCIVL, HIQL; same episodes and protocol as \cref{tab:main}, reported cell of each family). \wmpa{}-own scores member $i$'s imagined rollout with member $i$'s own value network; \wmpa{}-shared scores every rollout with the head of \cref{tab:main}. Bold marks values within $95\%$ of the row maximum.}
\label{tab:native-critic}
\centering\small
\input{tables/native_critic}
\end{table}

\begin{table}[t]
\caption{\textbf{Selecting the value head for datasets.} We report the concordance of each head with the ranking labels described in \cref{app:ablations}, that is, the fraction of state pairs on which the head assigns the higher value to the state that reaches the shared goal sooner. The last column indicates the head used in \cref{tab:main}, which was chosen per task family before any evaluation episode was run. An asterisk (${}^{*}$) indicates that the paired confidence interval excludes zero; in such cases, the head selected by the criterion is shown in bold.}
\label{tab:value-head-selection}
\centering
\small
\IfFileExists{tables/value_head_selection.tex}{\input{tables/value_head_selection}}{(pending)}
\end{table}

\begin{table}[htbp]
\caption{\textbf{Score ablations at the reported cell}: horizon aggregation and ensemble reduction of \cref{eq:score}; same protocol and episodes as \cref{tab:main}; on the puzzle family the ablations are of the direct value at the reported cell. Entries are mean success $\pm$ the half-width of the 95\% hierarchical-bootstrap interval (bank seeds, then episodes), rounded to integers; numbers within $95\%$ of the row maximum are in bold.}
\label{tab:ablations}
\centering
\scriptsize
\setlength{\tabcolsep}{3pt}
\resizebox{\linewidth}{!}{\input{tables/ablations}}
\end{table}

\section{Bank Composition}
\label[appendix]{app:bank-composition}

Every bank in the main study is trained on a single dataset. \Cref{tab:bank-composition} tests the next step: for an environment with two datasets, the six policies trained on its \texttt{play} data and the six trained on its \texttt{noisy} data are pooled into a single twelve-member bank, and \wmpa{} arbitrates over all twelve members with the same world model, value function, interval, and episodes as the six-member run on that dataset (nothing is retrained; the extra members are simply loaded from the checkpoints of the other dataset). This is the simplest case of a heterogeneous portfolio: same environment and interface, different training data. It isolates how arbitration handles members whose failure modes differ because their training data differ. For each base dataset, the table reports the best policy of the pooled bank and \wmpa{} over the pooled bank, and two paired contrasts: \wmpa{} versus the best policy of the pooled bank, and pooled-bank \wmpa{} versus six-member \wmpa{} on the same episodes. On \texttt{puzzle-4x4-play}, the best policy of the pooled bank is a foreign member, namely the GCIQL policy trained on the \texttt{noisy} data, which succeeds on \BankPooledBestPuzzleFourbyFourPlay\% of the \texttt{play} episodes, compared with \BankFullBestPuzzleFourbyFourPlay\% for the native one.

Where the added members contribute new competencies, the pooled bank outperforms the original one (\texttt{cube-double-play}: \BankDeltaFullPooledCubeDoublePlay); where they largely duplicate competencies already present, the pooled bank matches the original (\texttt{scene-play}: \BankDeltaFullPooledScenePlay; \texttt{puzzle-4x4-play}: \BankDeltaFullPooledPuzzleFourbyFourPlay), which indicates that the scorer is not misled by redundant or out-of-distribution candidates; moreover, on \texttt{puzzle-4x4-play}, \wmpa{} still gains \BankDeltaPooledPuzzleFourbyFourPlay{} points over a best policy that is itself foreign. We interpret this as evidence that \wmpa{} can be applied without re-tuning to banks assembled from policies trained on different data, provided that they share the same observation, action, and goal interfaces.

\begin{table}[htbp]
\caption{\textbf{Pooled play+noisy banks}: the six policies of the base dataset plus the six trained on its sibling dataset, arbitrated with the base dataset's world model, value, and interval on the base dataset's episodes. $P$: bank size; best policy of that bank and its success; \wmpa{} over that bank; paired contrasts of \wmpa{} against the bank's best policy and of pooled-bank \wmpa{} against the six-member \wmpa{} of the row above. Numbers within $95\%$ of the row maximum are in bold. (${}^{*}$: interval excludes zero).}
\label{tab:bank-composition}
\centering
\scriptsize
\resizebox{\linewidth}{!}{\input{tables/bank_composition}}
\end{table}

\section{World-Model Architecture and Training Details}
\label[appendix]{app:wm}

\paragraph{Data.} Each auxiliary model is trained on the OGBench training split of its dataset, with the benchmark's validation split used only for monitoring. Observation and one-step-delta statistics ($\mu_s,\sigma_s,\mu_\Delta,\sigma_\Delta$) are computed once from valid training transitions and stored in the checkpoint.

\paragraph{Dynamics ensemble.} $E=3$ members with disjoint parameters; each is an MLP with three hidden layers of width $512$, GELU activations, and layer normalization, mapping the normalized state and action to a normalized delta of the full observation vector (observation dimension $28$--$115$ depending on the dataset). Training unrolls each member for $H$ steps ($H=5$ for cube, scene, and puzzle; $H=10$ for the two mazes) with a curriculum that increases the unrolled horizon from $1$ to $H$ over the first $10^5$ steps. Members share the training batches and are diversified only by independent parameter initialization (and by unrolling their own predictions); at arbitration time value head $e$ scores the rollout of dynamics member $e$, and the ensemble is reduced by the mean over member pairs.

\paragraph{Metric value head.} For each ensemble member, a goal representation $\psi$ (an MLP $512\times3\rightarrow10$, length-normalized to $\sqrt{10}$) and state and goal encoders $\varphi_S,\varphi_G$ (each an MLP $512\times3\rightarrow64$) define $V=-\|\varphi_S(s)-\varphi_G(\psi(g))\|_2$, following LAVL \citep{kang2026lavl}; target networks are Polyak-averaged with $\tau=0.005$, and goals are relabeled with the OGBench value-goal mixture ($0.2$ current, $0.5$ geometric future, $0.3$ random). The per-family $\gamma$, $\kappa$, and $\lambda_{\mathrm{smooth}}$ are listed in \cref{tab:settings}. Where $\lambda_{\mathrm{smooth}}>0$ (mazes and scene), the local-smoothness regularizer softly bounds the one-step change of the value toward a random goal $\tilde g$,
\begin{equation}
\mathcal{L}_{\mathrm{smooth}}
= \E\Big[0.1\,\mathrm{softplus}\Big(10\big[\big(V_{\phi_e}(s',\tilde g)-V_{\phi_e}(s,\tilde g)\big)^2-\tau_V^2\big]\Big)\Big],
\qquad
\tau_V = 1-(1-\gamma)\,\E[\bar V],
\label{eq:smooth}
\end{equation}
where $\E[\bar V]$ is an exponential moving average of the mean value; this term is added to \cref{eq:expectile} with weight $\lambda_{\mathrm{smooth}}$. The expectile target is the implicit-maximization heuristic of IQL-style value learning \citep{kostrikov2022iql,park2023hiql}. Under stochastic transitions, it is also optimistic with respect to next-state noise, biasing the scorer; pessimistic ensemble reductions that counteract this bias change macro-average success by at most one point (\cref{app:scoreabl}).

\paragraph{Direct value head (puzzle).} The direct value $V_\theta$ follows the design of GCIQL, namely an MLP with three hidden layers of width $512$ and layer normalization, and is trained with the puzzle recipe of the reference implementation: the same relabeling mixture, $\gamma=0.99$, $\tau=0.005$, $\kappa=0.9$, and no value or target clipping. The two action-value heads $Q_{\omega,1},Q_{\omega,2}$ share this architecture and the training batches; $\bar Q$ in \cref{eq:direct-expectile} is the minimum of their target copies, and the same $\min_j Q_{\omega,j}$ is the scorer of the model-free Q-select baseline (\cref{tab:scorer-selector}).

\paragraph{Optimization.} The dynamics and value losses are summed with unit weights and optimized with Adam (learning rate $3\times10^{-4}$) at batch size $1024$ for $10^6$ steps on a single H100 MIG slice. The checkpoint at $10^6$ steps is used throughout. The same code path serves both training and test-time imagination, so that predicted states can be consumed directly by every frozen policy. The joint optimization described here applies to the dynamics and metric-value networks. The direct scorer is loaded from the pretrained GCIQL checkpoint and remains frozen during model training.
\section{Evaluation Protocol, Statistics, and the Static Per-Episode Oracle}
\label[appendix]{app:bank-stats}

\paragraph{Evaluation protocol and statistics.}
Every number follows the official OGBench evaluation: five goals per dataset, $50$ episodes per goal, deterministic actions ($250$ episodes per bank seed, $750$ per method).
Reset seeds are a deterministic function of the dataset, the goal, and the episode index, so every method in this paper, each bank policy, \wmpa{}, and every model-free and model-based control, is evaluated on the same episodes, and every contrast is paired per episode.
Uncertainty is a hierarchical bootstrap over bank seeds and then episodes ($10^4$ resamples, 95\% percentile intervals): intervals on success rates resample a method's episode outcomes, and intervals on contrasts resample the per-episode differences; a difference is significant when its interval excludes zero.

\Cref{tab:bank-stats} reports, for each dataset, the bank statistics introduced in \cref{sec:when}. Here, $B_{\max}$ denotes the success rate of the best policy, and $D_{\mathrm{gap}}$ denotes the performance gap between the best and second-best policies. Both quantities are computed from the scores in \cref{tab:main}. Because every bank member is evaluated on the same episodes as \wmpa{} and \randomswitch{}, namely \OracleEpisodes{} episodes per dataset ($5$ goals $\times$ \OracleEpsPerGoal{} episodes $\times$ $3$ bank seeds), the static per-episode oracle requires no additional runs. When we state that \wmpa{} exceeds the oracle, we count only the datasets on which it does so by at least one percentage point (\NumAboveUnion{} datasets). As a pre-specified criterion, we also report a paired hierarchical bootstrap confidence interval for the per-episode difference \wmpa$-U$ on the same episodes: the interval lies entirely above zero on \NumAboveUnionSig{} datasets.

\begin{table}[htbp]
\caption{\textbf{Bank statistics and the static per-episode oracle (\%)}. Here, $U$ is the success rate of a hindsight selector that picks a successful policy for each episode whenever one exists, and $U-B_{\max}$ is the headroom available to any per-episode selector that does not switch policies within an episode. The \wmpa{} column reports the scores from \cref{tab:main}, and $\Delta$\wmpa{} is the difference between \wmpa{} and $B_{\max}$. The \wmpa$-U$ column reports the paired per-episode difference between \wmpa{} and the oracle on the same episodes, together with a hierarchical bootstrap confidence interval computed over bank seeds and episodes. Among $B_{\max}$, $U$, and \wmpa{}, entries at or above $95\%$ of the row maximum are shown in bold; note that $U$ is a hindsight bound rather than a deployable method. The symbol ${}^{*}$ indicates that the confidence interval excludes zero.}
\label{tab:bank-stats}
\centering
\small
\setlength{\tabcolsep}{4pt}
\resizebox{\linewidth}{!}{\input{tables/bank_stats}}
\end{table}

\section{Selection Behavior and Runtime}
\label[appendix]{app:runtime}

\paragraph{Training cost.} On the same GPU slice (an H100 3g.40gb MIG partition), training the world model together with its value head takes $3.3$ hours per dataset on average, compared with $2.7$ hours for a single policy and approximately $16$ hours for a six-policy bank, a cost that the best-policy baseline also incurs because all candidates must be trained before one is kept. Test-time latency grows linearly with the bank size $M$, because the arbiter makes $ME$ model calls per environment step: the twelve-member pooled banks of \cref{app:bank-composition} run at $5$--$6$\,ms per step, roughly twice the per-arbitration latency of a six-member bank.

\paragraph{Arbitration behavior.} \Cref{tab:runtime} reports how the portfolio is used: episode length, switches per episode for \wmpa{} and \randomswitch, arbitration decisions per episode, world-model calls per environment step (dynamics; value calls are identical), the entropy of the policy-usage distribution (bits; $\log_2 6\approx2.58$ for uniform usage of a six-member bank), and wall-clock latency per environment step. Latency was measured on the evaluation nodes (AMD EPYC 9135, 8 threads, JAX on CPU, no GPU) and includes batched policy inference, dynamics and value evaluation, and the winner's action. \wmpa{} switches less often than \randomswitch{} at the same interval (it frequently re-selects the incumbent), and its usage entropy is high on the datasets with the largest gains, showing that the arbiter does not collapse to a single bank member. Evidence that within-episode composition improves success is provided separately by the oracle comparison in \cref{app:bank-stats}.
 
\begin{table}[htbp]
\caption{Selection behavior and runtime of \wmpa{} (selected $(k,c)$) and \randomswitch{} (same $c$). ep.\ len and decisions/ep belong to \wmpa{}.}
\label{tab:runtime}
\centering
\scriptsize
\setlength{\tabcolsep}{3pt}
\resizebox{\linewidth}{!}{\input{tables/runtime}}
\end{table}

\section{When Do Imagined Rollouts Support Useful Arbitration?}
\label[appendix]{app:wmpa_theory}

Fix a commitment interval $c$, and let $\pi_b$ be the fixed bank member
with the highest success probability under the evaluation distribution.
Write $J(\pi)$ for the probability that policy $\pi$ reaches the goal by
the deadline $T$. A decision state $x$ includes the environment state,
goal, remaining time, and success or termination status. Success and
termination are absorbing, and the final commitment interval is shortened
to the remaining episode length. There are $N=\lceil T/c\rceil$ decision
slots, with absorbing trajectories padded to the deadline.

The \emph{chunk advantage} of member $i$ relative to $\pi_b$ is
\begin{equation}
    A_i(x)
    = \Pr\!\left(\text{success}\mid x,\
        \pi_i\text{ for one chunk, then }\pi_b\right)
    - \Pr\!\left(\text{success}\mid x,\
        \pi_b\text{ throughout}\right).
    \label{eq:wmpa_chunk_advantage}
\end{equation}
This quantity measures the benefit of using member $i$ for the next
chunk while holding the subsequent controller fixed. Since $A_b(x)=0$,
the largest available advantage,
\begin{equation}
    G(x)=\max_i A_i(x)\geq 0,
    \label{eq:wmpa_opportunity}
\end{equation}
measures the opportunity to improve on the best fixed member at that
decision point.

Let $S_i(x)$ be the imagined-rollout score used by \wmpa{}, and let
$S_i^\star(x)$ be the expected value of the same complete scoring rule
under the true dynamics, with the learned value head unchanged. In
particular, the expectation is taken after applying the rule's
aggregation over ensemble members and maximum over rollout steps. Any
sampling error in the computed score $S_i(x)$ is included in the
discrepancy below. Define
\begin{equation}
    e(x)=\max_i\left|S_i(x)-S_i^\star(x)\right|.
    \label{eq:wmpa_score_discrepancy}
\end{equation}
Thus, $e(x)$ measures the largest discrepancy between the computed and
true-dynamics scores across the bank. After success or termination, set all
advantages and scores to zero.

\begin{wmpaassumption}[Separation of inferior choices]
\label{ass:wmpa_score_separation}
For some tolerance $0\leq\delta<1$ and margin $\Delta>0$, at every
active decision state visited by \wmpa{},
\begin{equation}
    A_i(x)<G(x)-\delta
    \quad\Longrightarrow\quad
    S_i^\star(x)\leq\max_j S_j^\star(x)-\Delta.
    \label{eq:wmpa_score_separation}
\end{equation}
\end{wmpaassumption}

The assumption requires the value head, when applied to true-dynamics
rollouts, to separate any member whose advantage is more than $\delta$
below the best by a score margin of at least $\Delta$. It does not require
correct ordering among members whose advantages are within $\delta$ of
the best. The tolerance is measured in success-probability units, whereas
the margin is measured in the scorer's own units.

\begin{wmpatheorem}[Improvement through robust selection]
\label{thm:wmpa_robust_selection}
Under Assumption~\ref{ass:wmpa_score_separation}, let $X_n$ be the decision
state at slot $n$ under \wmpa{}, and let $\mathcal A_n$ be the event that
$X_n$ is active. Then
\begin{equation}
\begin{split}
    J(\pi_{\mathrm{WMPA}})-J(\pi_b)
    \geq \sum_{n=0}^{N-1}
    \biggl(&\mathbb E_{\mathrm{WMPA}}[G(X_n)]
    -\delta\Pr_{\mathrm{WMPA}}(\mathcal A_n)\\
    &-(1-\delta)\Pr_{\mathrm{WMPA}}\!\left(
        \mathcal A_n,\,
        e(X_n)\geq\frac{\Delta}{2}\right)\biggr).
\end{split}
\label{eq:wmpa_robust_improvement}
\end{equation}
In particular, a positive right-hand side guarantees improvement over
every fixed bank member.
\end{wmpatheorem}

The bound compares the available improvement with two sources of
selection loss at active decision slots: the tolerated difference $\delta$
from the best chunk advantage, and decisions at which score error is
large enough to overturn the separation margin. Errors smaller than half
the margin cannot cause a selection more than $\delta$ below the best
advantage.

\begin{proof}
At an active decision state $x$, let $\hat i$ maximize the computed score
and let $j$ maximize the true-dynamics score. Since
$S_{\hat i}(x)\geq S_j(x)$,
\begin{equation}
    S_j^\star(x)-S_{\hat i}^\star(x)\leq 2e(x).
    \label{eq:wmpa_margin_stability}
\end{equation}
If $e(x)<\Delta/2$, Assumption~\ref{ass:wmpa_score_separation} implies
$A_{\hat i}(x)\geq G(x)-\delta$. Otherwise,
$G(x)-A_{\hat i}(x)\leq 1$, since the common reference term cancels and
the remaining difference is between two success probabilities.
At an absorbing decision state, all advantages are zero. Consequently,
for every decision state $x$,
\begin{equation}
    A_{\hat i}(x)
    \geq G(x)-\delta\mathbf{1}\{x\text{ is active}\}
    -(1-\delta)\mathbf{1}\!\left\{
        x\text{ is active},\ e(x)\geq\frac{\Delta}{2}\right\}.
    \label{eq:wmpa_local_selection_bound}
\end{equation}

Let $V_b(x)$ denote the probability of success if $\pi_b$ is followed
from $x$ onward. By definition, the expected change in $V_b$ across a
chunk executed by member $i$ is $A_i(x)$. Summing these changes along
\wmpa{}'s trajectory cancels the intermediate reference values. At the
deadline, $V_b$ equals the success indicator; its expectation at the
initial state equals $J(\pi_b)$. Therefore,
\begin{equation}
    J(\pi_{\mathrm{WMPA}})-J(\pi_b)
    =\sum_{n=0}^{N-1}
    \mathbb E_{\mathrm{WMPA}}[A_{\hat i_n}(X_n)],
    \label{eq:wmpa_advantage_identity}
\end{equation}
where $\hat i_n$ is the member selected at slot $n$. Taking expectations
in~\eqref{eq:wmpa_local_selection_bound} and summing proves the result.
\end{proof}

The result identifies three requirements for useful arbitration: the bank
must offer locally advantageous alternatives, the value head must
separate substantially inferior choices under true dynamics, and score
errors must be small relative to that separation. Accurate dynamics
alone are insufficient if the value head favors unhelpful outcomes.
Conversely, imperfect imagined rollouts can support reliable selection
when their score errors preserve the relevant margin.
\end{document}

%% file: tables/numbers.tex
\newcommand{\NumEnvs}{18}

\newcommand{\MeanWMPP}{58}
\newcommand{\MeanBest}{44}
\newcommand{\MeanRandom}{39}
\newcommand{\MeanGain}{13}
\newcommand{\MeanWMPPminusRandom}{19}

\newcommand{\NumSigUp}{12}
\newcommand{\NumSigDown}{1}
\newcommand{\NumNS}{5}
\newcommand{\NumWMPPvsRandomSig}{13}

\newcommand{\NumRandomSigUp}{4}
\newcommand{\NumRandomSigDown}{10}

\newcommand{\NumAboveUnion}{10}

\newcommand{\NumAboveUnionSig}{7}
\newcommand{\NumBelowUnionSig}{5}

\newcommand{\NumBetterMean}{17}

\newcommand{\OracleEpsPerGoal}{50}

\newcommand{\OracleEpisodes}{750}

\newcommand{\GainCubeSinglePlay}{+14}

\newcommand{\DgapCubeSinglePlay}{22}

\newcommand{\HeadroomCubeSinglePlay}{10}

\newcommand{\GainCubeDoublePlay}{+33}
\newcommand{\WmppCubeDoublePlay}{69}

\newcommand{\UnionCubeDoublePlay}{57}

\newcommand{\GainCubeDoubleNoisy}{+35}

\newcommand{\GainScenePlay}{+36}
\newcommand{\WmppScenePlay}{87}

\newcommand{\UnionScenePlay}{74}

\newcommand{\GainSceneNoisy}{+41}

\newcommand{\GainPuzzleThreebyThreePlay}{+9}

\newcommand{\DgapPuzzleThreebyThreePlay}{79}

\newcommand{\HeadroomPuzzleThreebyThreePlay}{1}

%% file: tables/numbers_ablations.tex
\newcommand{\AblMeanDeltaAggLast}{+1}

\newcommand{\AblMeanDeltaAggMean}{-1}

\newcommand{\AblMeanDeltaEnsMin}{+0}

\newcommand{\AblMeanDeltaEnsLcb}{+0}

\newcommand{\AblMeanPolicyMpc}{53}

\newcommand{\AblNumDownPolicyMpc}{7}
\newcommand{\AblNumUpPolicyMpc}{1}
\newcommand{\AblNumEnvsPolicyMpc}{18}

\newcommand{\AblMeanDeltaSimSel}{+2}

\newcommand{\AblNumUpSimSel}{2}
\newcommand{\AblNumEnvsSimSel}{18}

\newcommand{\NumStallBeatsWmpp}{2}

\newcommand{\NumEnvsStall}{18}

\newcommand{\BankDeltaFullPooledCubeDoublePlay}{+10 [+6, +15]*}

\newcommand{\BankDeltaFullPooledScenePlay}{+1 [-3, +5]}

\newcommand{\BankFullBestPuzzleFourbyFourPlay}{24}

\newcommand{\BankPooledBestPuzzleFourbyFourPlay}{34}
\newcommand{\BankDeltaPooledPuzzleFourbyFourPlay}{+18 [+11, +26]*}
\newcommand{\BankDeltaFullPooledPuzzleFourbyFourPlay}{+2 [-4, +8]}

%% file: tables/numbers_mechanism.tex
\newcommand{\MechNumEnvsQselOne}{18}

\newcommand{\MechNumAboveBestQselOne}{10}

\newcommand{\MechNumAboveBestQselK}{5}

\newcommand{\MechNumEnvsDirectQselK}{18}

\newcommand{\MechNumUpDirectQselK}{11}

\newcommand{\MechMeanDeltaDirectQselK}{+7}

\newcommand{\MechMeanDeltaFullQselOne}{+12}

\newcommand{\MechFamFullQselOnePuzzle}{+2}

%% file: tables/numbers_kcgrid.tex
\newcommand{\KcGridMinGain}{26}

\newcommand{\KcGridNumCells}{18}

%% file: tables/main_results.tex
\begin{tabular}{llcccccccc}
\toprule
Family & Dataset & GCBC & GCIVL & GCIQL & QRL & CRL & HIQL & \bestfixed{} & \wmpp{} \\
\midrule
\multirow{2}{*}{Maze} & \texttt{pointmaze-medium-navigate} & $6 \pm 4$ & $74 \pm 3$ & $54 \pm 16$ & $\mathbf{76 \pm 17}$ & $35 \pm 6$ & $73 \pm 8$ & $\mathbf{76 \pm 17}$ & $\mathbf{79 \pm 3}$ \\
 & \texttt{antmaze-large-navigate} & $23 \pm 3$ & $19 \pm 5$ & $37 \pm 8$ & $78 \pm 5$ & $\mathbf{86 \pm 3}$ & $\mathbf{90 \pm 3}$ & $\mathbf{90 \pm 3}$ & $\mathbf{91 \pm 3}$ \\
\midrule
\multirow{6}{*}{Cube} & \texttt{cube-single-play} & $6 \pm 3$ & $49 \pm 4$ & $71 \pm 3$ & $8 \pm 4$ & $16 \pm 3$ & $12 \pm 6$ & $71 \pm 3$ & $\mathbf{85 \pm 3}$ \\
 & \texttt{cube-single-noisy} & $8 \pm 3$ & $71 \pm 8$ & $\mathbf{99 \pm 1}$ & $33 \pm 10$ & $40 \pm 5$ & $40 \pm 10$ & $\mathbf{99 \pm 1}$ & $\mathbf{100 \pm 0}$ \\
 & \texttt{cube-double-play} & $2 \pm 1$ & $35 \pm 6$ & $36 \pm 7$ & $1 \pm 1$ & $9 \pm 2$ & $5 \pm 2$ & $36 \pm 7$ & $\mathbf{69 \pm 3}$ \\
 & \texttt{cube-double-noisy} & $1 \pm 1$ & $14 \pm 3$ & $26 \pm 9$ & $3 \pm 2$ & $3 \pm 2$ & $2 \pm 1$ & $26 \pm 9$ & $\mathbf{60 \pm 7}$ \\
 & \texttt{cube-triple-play} & $1 \pm 1$ & $1 \pm 1$ & $2 \pm 1$ & $0 \pm 0$ & $\mathbf{5 \pm 3}$ & $2 \pm 1$ & $\mathbf{5 \pm 3}$ & $1 \pm 1$ \\
 & \texttt{cube-triple-noisy} & $2 \pm 2$ & $9 \pm 3$ & $2 \pm 2$ & $2 \pm 1$ & $1 \pm 1$ & $1 \pm 1$ & $9 \pm 3$ & $\mathbf{18 \pm 3}$ \\
\midrule
\multirow{2}{*}{Scene} & \texttt{scene-play} & $4 \pm 2$ & $42 \pm 5$ & $52 \pm 5$ & $6 \pm 2$ & $22 \pm 4$ & $43 \pm 8$ & $52 \pm 5$ & $\mathbf{87 \pm 4}$ \\
 & \texttt{scene-noisy} & $2 \pm 1$ & $30 \pm 3$ & $27 \pm 6$ & $6 \pm 3$ & $0 \pm 1$ & $26 \pm 4$ & $30 \pm 3$ & $\mathbf{70 \pm 5}$ \\
\midrule
\multirow{8}{*}{Puzzle} & \texttt{puzzle-3x3-play} & $2 \pm 1$ & $3 \pm 1$ & $91 \pm 4$ & $0 \pm 1$ & $4 \pm 2$ & $12 \pm 3$ & $91 \pm 4$ & $\mathbf{100 \pm 0}$ \\
 & \texttt{puzzle-3x3-noisy} & $1 \pm 1$ & $38 \pm 26$ & $\mathbf{91 \pm 3}$ & $0 \pm 0$ & $31 \pm 5$ & $57 \pm 7$ & $\mathbf{91 \pm 3}$ & $\mathbf{93 \pm 2}$ \\
 & \texttt{puzzle-4x4-play} & $0 \pm 0$ & $12 \pm 5$ & $24 \pm 5$ & $0 \pm 0$ & $1 \pm 1$ & $7 \pm 2$ & $24 \pm 5$ & $\mathbf{51 \pm 8}$ \\
 & \texttt{puzzle-4x4-noisy} & $0 \pm 0$ & $19 \pm 3$ & $34 \pm 5$ & $0 \pm 0$ & $0 \pm 0$ & $12 \pm 4$ & $34 \pm 5$ & $\mathbf{58 \pm 11}$ \\
 & \texttt{puzzle-4x5-play} & $0 \pm 0$ & $6 \pm 2$ & $14 \pm 3$ & $0 \pm 0$ & $2 \pm 1$ & $4 \pm 2$ & $14 \pm 3$ & $\mathbf{18 \pm 3}$ \\
 & \texttt{puzzle-4x5-noisy} & $0 \pm 0$ & $\mathbf{20 \pm 3}$ & $\mathbf{20 \pm 3}$ & $0 \pm 0$ & $5 \pm 4$ & $3 \pm 2$ & $\mathbf{20 \pm 3}$ & $\mathbf{20 \pm 3}$ \\
 & \texttt{puzzle-4x6-play} & $0 \pm 0$ & $11 \pm 2$ & $11 \pm 2$ & $0 \pm 0$ & $3 \pm 1$ & $3 \pm 2$ & $11 \pm 2$ & $\mathbf{17 \pm 3}$ \\
 & \texttt{puzzle-4x6-noisy} & $0 \pm 0$ & $18 \pm 3$ & $\mathbf{19 \pm 3}$ & $0 \pm 0$ & $5 \pm 2$ & $2 \pm 1$ & $\mathbf{19 \pm 3}$ & $\mathbf{20 \pm 3}$ \\
\midrule
\multicolumn{2}{l}{Average (18 datasets)} & $3$ & $26$ & $39$ & $12$ & $15$ & $22$ & $44$ & $\mathbf{58}$ \\
\bottomrule
\end{tabular}

%% file: tables/settings_family.tex
\begin{tabular}{lclccccc}
\toprule
Family & datasets & value head & $k{=}c$ & WM $H$ & $\gamma$ & $\kappa$ & $\lambda_{\mathrm{smooth}}$ \\
\midrule
Maze & 2 & metric & 1 & 10 & 0.999 & 0.9 & 10 \\
Cube & 6 & metric & 5 & 5 & 0.99 & 0.7 & 0 \\
Scene & 2 & metric & 10 & 5 & 0.998 & 0.7 & 10 \\
Puzzle & 8 & direct & 10 & 5 & 0.99 & 0.9 & -- \\
\bottomrule
\end{tabular}

%% file: tables/kc_grid.tex
\begin{tabular}{lcl|ccc|ccc|ccc}
\toprule
 & & & \multicolumn{3}{c|}{success rate (\%)} & \multicolumn{3}{c|}{switches per episode} & \multicolumn{3}{c}{latency (ms per step)} \\
Dataset & \bestfixed{} & $k$ & $c{=}1$ & $c{=}5$ & $c{=}10$ & $c{=}1$ & $c{=}5$ & $c{=}10$ & $c{=}1$ & $c{=}5$ & $c{=}10$ \\
\midrule
\multirow{3}{*}{\texttt{cube-double-play}} & \multirow{3}{*}{$36 \pm 7$} & $1$ & $63 \pm 4$ &  &  & $147$ &  &  & $3.9$ &  &  \\
 &  & $5$ & $65 \pm 4$ & \underline{$69 \pm 3$} &  & $107$ & $38$ &  & $15.5$ & $3.0$ &  \\
 &  & $10$ & $64 \pm 3$ & $75 \pm 3$ & $73 \pm 4$ & $90$ & $30$ & $19$ & $27.6$ & $5.9$ & $2.9$ \\
\midrule
\multirow{3}{*}{\texttt{scene-play}} & \multirow{3}{*}{$52 \pm 5$} & $1$ & $93 \pm 3$ &  &  & $135$ &  &  & $4.1$ &  &  \\
 &  & $5$ & $93 \pm 2$ & $92 \pm 2$ &  & $94$ & $35$ &  & $15.5$ & $3.0$ &  \\
 &  & $10$ & $89 \pm 3$ & $93 \pm 3$ & \underline{$87 \pm 4$} & $85$ & $29$ & $21$ & $31.8$ & $6.6$ & $3.0$ \\
\midrule
\multirow{3}{*}{\texttt{puzzle-4x4-play}} & \multirow{3}{*}{$24 \pm 5$} & $1$ & $50 \pm 6$ &  &  & $231$ &  &  & $16.8$ &  &  \\
 &  & $5$ & $52 \pm 11$ & $53 \pm 6$ &  & $167$ & $56$ &  & $23.0$ & $5.9$ &  \\
 &  & $10$ & $55 \pm 8$ & $54 \pm 10$ & \underline{$51 \pm 8$} & $131$ & $49$ & $29$ & $41.5$ & $6.7$ & $4.5$ \\
\bottomrule
\end{tabular}

%% file: tables/random_by_c.tex
\begin{tabular}{lcccc}
\toprule
Dataset & $c$ & Best & \randomswitch & \wmpp \\
\midrule
\texttt{pointmaze-medium-navigate} & $1$ & $76 \pm 17$ & $\mathbf{89 \pm 13}$ & $79 \pm 3$ \\
\texttt{antmaze-large-navigate} & $1$ & $\mathbf{90 \pm 3}$ & $\mathbf{92 \pm 2}$ & $\mathbf{91 \pm 3}$ \\
\texttt{cube-single-play} & $5$ & $71 \pm 3$ & $41 \pm 4$ & $\mathbf{85 \pm 3}$ \\
\texttt{cube-single-noisy} & $5$ & $\mathbf{99 \pm 1}$ & $\mathbf{96 \pm 2}$ & $\mathbf{100 \pm 0}$ \\
\texttt{cube-double-play} & $5$ & $36 \pm 7$ & $24 \pm 4$ & $\mathbf{69 \pm 3}$ \\
\texttt{cube-double-noisy} & $5$ & $26 \pm 9$ & $27 \pm 4$ & $\mathbf{60 \pm 7}$ \\
\texttt{cube-triple-play} & $5$ & $\mathbf{5 \pm 3}$ & $\mathbf{5 \pm 3}$ & $1 \pm 1$ \\
\texttt{cube-triple-noisy} & $5$ & $9 \pm 3$ & $\mathbf{25 \pm 3}$ & $18 \pm 3$ \\
\texttt{scene-play} & $10$ & $52 \pm 5$ & $42 \pm 4$ & $\mathbf{87 \pm 4}$ \\
\texttt{scene-noisy} & $10$ & $30 \pm 3$ & $47 \pm 5$ & $\mathbf{70 \pm 5}$ \\
\texttt{puzzle-3x3-play} & $10$ & $91 \pm 4$ & $21 \pm 4$ & $\mathbf{100 \pm 0}$ \\
\texttt{puzzle-3x3-noisy} & $10$ & $91 \pm 3$ & $\mathbf{97 \pm 2}$ & $\mathbf{93 \pm 2}$ \\
\texttt{puzzle-4x4-play} & $10$ & $24 \pm 5$ & $6 \pm 3$ & $\mathbf{51 \pm 8}$ \\
\texttt{puzzle-4x4-noisy} & $10$ & $34 \pm 5$ & $49 \pm 6$ & $\mathbf{58 \pm 11}$ \\
\texttt{puzzle-4x5-play} & $10$ & $14 \pm 3$ & $4 \pm 2$ & $\mathbf{18 \pm 3}$ \\
\texttt{puzzle-4x5-noisy} & $10$ & $\mathbf{20 \pm 3}$ & $13 \pm 3$ & $\mathbf{20 \pm 3}$ \\
\texttt{puzzle-4x6-play} & $10$ & $11 \pm 2$ & $5 \pm 2$ & $\mathbf{17 \pm 3}$ \\
\texttt{puzzle-4x6-noisy} & $10$ & $\mathbf{19 \pm 3}$ & $15 \pm 3$ & $\mathbf{20 \pm 3}$ \\
\bottomrule
\end{tabular}

%% file: tables/stall.tex
\begin{tabular}{lc|cccc|c}
\toprule
 & & \multicolumn{4}{c|}{stall-restart, window $m$} & \\
Dataset & Best & $m{=}5$ & $10$ & $25$ & $50$ & \wmpp \\
\midrule
\texttt{pointmaze-medium-navigate} & $76 \pm 17$ & $\mathbf{99 \pm 1}$ & $\mathbf{95 \pm 2}$ & $\mathbf{94 \pm 2}$ & $\mathbf{96 \pm 2}$ & $79 \pm 3$ \\
\texttt{antmaze-large-navigate} & $\mathbf{90 \pm 3}$ & $51 \pm 4$ & $51 \pm 4$ & $51 \pm 4$ & $51 \pm 4$ & $\mathbf{91 \pm 3}$ \\
\texttt{cube-single-play} & $71 \pm 3$ & $24 \pm 5$ & $24 \pm 5$ & $24 \pm 5$ & $24 \pm 5$ & $\mathbf{85 \pm 3}$ \\
\texttt{cube-single-noisy} & $\mathbf{99 \pm 1}$ & $53 \pm 5$ & $52 \pm 5$ & $52 \pm 5$ & $52 \pm 5$ & $\mathbf{100 \pm 0}$ \\
\texttt{cube-double-play} & $36 \pm 7$ & $16 \pm 4$ & $15 \pm 4$ & $15 \pm 4$ & $15 \pm 4$ & $\mathbf{69 \pm 3}$ \\
\texttt{cube-double-noisy} & $26 \pm 9$ & $11 \pm 2$ & $12 \pm 3$ & $11 \pm 3$ & $11 \pm 3$ & $\mathbf{60 \pm 7}$ \\
\texttt{cube-triple-play} & $\mathbf{5 \pm 3}$ & $2 \pm 1$ & $3 \pm 2$ & $2 \pm 1$ & $2 \pm 1$ & $1 \pm 1$ \\
\texttt{cube-triple-noisy} & $9 \pm 3$ & $4 \pm 2$ & $5 \pm 2$ & $4 \pm 1$ & $5 \pm 2$ & $\mathbf{18 \pm 3}$ \\
\texttt{scene-play} & $52 \pm 5$ & $32 \pm 4$ & $31 \pm 5$ & $30 \pm 4$ & $30 \pm 4$ & $\mathbf{87 \pm 4}$ \\
\texttt{scene-noisy} & $30 \pm 3$ & $19 \pm 3$ & $20 \pm 3$ & $18 \pm 3$ & $20 \pm 4$ & $\mathbf{70 \pm 5}$ \\
\texttt{puzzle-3x3-play} & $91 \pm 4$ & $20 \pm 3$ & $20 \pm 3$ & $20 \pm 3$ & $20 \pm 3$ & $\mathbf{100 \pm 0}$ \\
\texttt{puzzle-3x3-noisy} & $\mathbf{91 \pm 3}$ & $44 \pm 6$ & $45 \pm 5$ & $43 \pm 6$ & $44 \pm 6$ & $\mathbf{93 \pm 2}$ \\
\texttt{puzzle-4x4-play} & $24 \pm 5$ & $8 \pm 2$ & $8 \pm 2$ & $8 \pm 2$ & $8 \pm 2$ & $\mathbf{51 \pm 8}$ \\
\texttt{puzzle-4x4-noisy} & $34 \pm 5$ & $12 \pm 3$ & $12 \pm 3$ & $11 \pm 2$ & $11 \pm 2$ & $\mathbf{58 \pm 11}$ \\
\texttt{puzzle-4x5-play} & $14 \pm 3$ & $4 \pm 2$ & $4 \pm 1$ & $4 \pm 1$ & $4 \pm 2$ & $\mathbf{18 \pm 3}$ \\
\texttt{puzzle-4x5-noisy} & $\mathbf{20 \pm 3}$ & $7 \pm 2$ & $8 \pm 2$ & $8 \pm 2$ & $8 \pm 2$ & $\mathbf{20 \pm 3}$ \\
\texttt{puzzle-4x6-play} & $11 \pm 2$ & $5 \pm 2$ & $5 \pm 2$ & $5 \pm 2$ & $5 \pm 2$ & $\mathbf{17 \pm 3}$ \\
\texttt{puzzle-4x6-noisy} & $\mathbf{19 \pm 3}$ & $10 \pm 2$ & $11 \pm 3$ & $9 \pm 2$ & $10 \pm 2$ & $\mathbf{20 \pm 3}$ \\
\bottomrule
\end{tabular}

%% file: tables/search.tex
\begin{tabular}{lccc|c}
\toprule
Dataset & $(k,c)$ of \wmpp & Best & every-step WM action search on Best & \wmpp \\
\midrule
\texttt{pointmaze-medium-navigate} & $(1,1)$ & $76 \pm 17$ & $\mathbf{86 \pm 9}$ & $79 \pm 3$ \\
\texttt{antmaze-large-navigate} & $(1,1)$ & $\mathbf{90 \pm 3}$ & $\mathbf{92 \pm 3}$ & $\mathbf{91 \pm 3}$ \\
\texttt{cube-single-play} & $(5,5)$ & $71 \pm 3$ & $52 \pm 4$ & $\mathbf{85 \pm 3}$ \\
\texttt{cube-single-noisy} & $(5,5)$ & $\mathbf{99 \pm 1}$ & $\mathbf{100 \pm 0}$ & $\mathbf{100 \pm 0}$ \\
\texttt{cube-double-play} & $(5,5)$ & $36 \pm 7$ & $39 \pm 4$ & $\mathbf{69 \pm 3}$ \\
\texttt{cube-double-noisy} & $(5,5)$ & $26 \pm 9$ & $49 \pm 6$ & $\mathbf{60 \pm 7}$ \\
\texttt{cube-triple-play} & $(5,5)$ & $\mathbf{5 \pm 3}$ & $0 \pm 1$ & $1 \pm 1$ \\
\texttt{cube-triple-noisy} & $(5,5)$ & $9 \pm 3$ & $\mathbf{21 \pm 3}$ & $18 \pm 3$ \\
\texttt{scene-play} & $(10,10)$ & $52 \pm 5$ & $69 \pm 4$ & $\mathbf{87 \pm 4}$ \\
\texttt{scene-noisy} & $(10,10)$ & $30 \pm 3$ & $65 \pm 4$ & $\mathbf{70 \pm 5}$ \\
\texttt{puzzle-3x3-play} & $(10,10)$ & $91 \pm 4$ & $\mathbf{97 \pm 2}$ & $\mathbf{100 \pm 0}$ \\
\texttt{puzzle-3x3-noisy} & $(10,10)$ & $\mathbf{91 \pm 3}$ & $\mathbf{92 \pm 2}$ & $\mathbf{93 \pm 2}$ \\
\texttt{puzzle-4x4-play} & $(10,10)$ & $24 \pm 5$ & $\mathbf{54 \pm 9}$ & $51 \pm 8$ \\
\texttt{puzzle-4x4-noisy} & $(10,10)$ & $34 \pm 5$ & $52 \pm 11$ & $\mathbf{58 \pm 11}$ \\
\texttt{puzzle-4x5-play} & $(10,10)$ & $14 \pm 3$ & $\mathbf{20 \pm 3}$ & $18 \pm 3$ \\
\texttt{puzzle-4x5-noisy} & $(10,10)$ & $\mathbf{20 \pm 3}$ & $\mathbf{20 \pm 3}$ & $\mathbf{20 \pm 3}$ \\
\texttt{puzzle-4x6-play} & $(10,10)$ & $11 \pm 2$ & $\mathbf{20 \pm 3}$ & $17 \pm 3$ \\
\texttt{puzzle-4x6-noisy} & $(10,10)$ & $\mathbf{19 \pm 3}$ & $\mathbf{20 \pm 3}$ & $\mathbf{20 \pm 3}$ \\
\bottomrule
\end{tabular}

%% file: tables/sim_all.tex
\begin{tabular}{lcccc r}
\toprule
Dataset & $(k,c)$ & Best & \wmpp & true-sim.\ branches & $\Delta$ (sim $-$ \wmpp) \\
\midrule
\texttt{pointmaze-medium-navigate} & $(1,1)$ & $\mathbf{76 \pm 17}$ & $\mathbf{79 \pm 3}$ & $\mathbf{80 \pm 3}$ & +0 [-0, +1] \\
\texttt{antmaze-large-navigate} & $(1,1)$ & $90 \pm 3$ & $\mathbf{91 \pm 3}$ & $\mathbf{95 \pm 2}$ & +4 [-0, +9] \\
\texttt{cube-single-play} & $(5,5)$ & $71 \pm 3$ & $\mathbf{85 \pm 3}$ & $\mathbf{86 \pm 3}$ & +2 [-3, +7] \\
\texttt{cube-single-noisy} & $(5,5)$ & $\mathbf{99 \pm 1}$ & $\mathbf{100 \pm 0}$ & $\mathbf{100 \pm 0}$ & +0 [+0, +1] \\
\texttt{cube-double-play} & $(5,5)$ & $36 \pm 7$ & $\mathbf{69 \pm 3}$ & $\mathbf{72 \pm 4}$ & +3 [-1, +7] \\
\texttt{cube-double-noisy} & $(5,5)$ & $26 \pm 9$ & $\mathbf{60 \pm 7}$ & $\mathbf{61 \pm 4}$ & +0 [-10, +9] \\
\texttt{cube-triple-play} & $(5,5)$ & $\mathbf{5 \pm 3}$ & $1 \pm 1$ & $0 \pm 1$ & -0 [-1, +0] \\
\texttt{cube-triple-noisy} & $(5,5)$ & $9 \pm 3$ & $\mathbf{18 \pm 3}$ & $\mathbf{18 \pm 4}$ & +0 [-3, +4] \\
\texttt{scene-play} & $(10,10)$ & $52 \pm 5$ & $87 \pm 4$ & $\mathbf{95 \pm 3}$ & +7 [+5, +10]* \\
\texttt{scene-noisy} & $(10,10)$ & $30 \pm 3$ & $70 \pm 5$ & $\mathbf{77 \pm 6}$ & +7 [+4, +11]* \\
\texttt{puzzle-3x3-play} & $(10,10)$ & $91 \pm 4$ & $\mathbf{100 \pm 0}$ & $\mathbf{100 \pm 0}$ & -0 [-1, +0] \\
\texttt{puzzle-3x3-noisy} & $(10,10)$ & $\mathbf{91 \pm 3}$ & $\mathbf{93 \pm 2}$ & $\mathbf{93 \pm 3}$ & +0 [-2, +3] \\
\texttt{puzzle-4x4-play} & $(10,10)$ & $24 \pm 5$ & $\mathbf{51 \pm 8}$ & $\mathbf{53 \pm 6}$ & +2 [-4, +8] \\
\texttt{puzzle-4x4-noisy} & $(10,10)$ & $34 \pm 5$ & $58 \pm 11$ & $\mathbf{62 \pm 11}$ & +4 [-0, +7] \\
\texttt{puzzle-4x5-play} & $(10,10)$ & $14 \pm 3$ & $\mathbf{18 \pm 3}$ & $\mathbf{18 \pm 3}$ & +0 [-1, +1] \\
\texttt{puzzle-4x5-noisy} & $(10,10)$ & $\mathbf{20 \pm 3}$ & $\mathbf{20 \pm 3}$ & $\mathbf{20 \pm 3}$ & +0 [+0, +0] \\
\texttt{puzzle-4x6-play} & $(10,10)$ & $11 \pm 2$ & $\mathbf{17 \pm 3}$ & $\mathbf{18 \pm 3}$ & +1 [-1, +3] \\
\texttt{puzzle-4x6-noisy} & $(10,10)$ & $\mathbf{19 \pm 3}$ & $\mathbf{20 \pm 3}$ & $\mathbf{20 \pm 3}$ & -0 [-1, +0] \\
\bottomrule
\end{tabular}

%% file: tables/scorer_selector.tex
\begin{tabular}{lcc|cc|cc|r}
\toprule
 & & & \multicolumn{2}{c|}{\wmpp{} (imagined states)} & \multicolumn{2}{c|}{Q-select (no rollout)} & \\
\cmidrule(lr){4-5}\cmidrule(lr){6-7}
Dataset & $(k,c)$ & Best & metric value & direct value & $c{=}1$ & $c{=}k$ & $\Delta$ (direct $-$ Q-select, $c{=}k$) \\
\midrule
\texttt{pointmaze-medium-navigate} & $(1,1)$ & $\mathbf{76 \pm 17}$ & $\mathbf{79 \pm 3}$$^{\dagger}$ & $56 \pm 8$ & $53 \pm 9$ & $53 \pm 9$ & +3 [+1, +4]* \\
\texttt{antmaze-large-navigate} & $(1,1)$ & $\mathbf{90 \pm 3}$ & $\mathbf{91 \pm 3}$$^{\dagger}$ & $28 \pm 4$ & $26 \pm 6$ & $26 \pm 6$ & +2 [-2, +6] \\
\texttt{cube-single-play} & $(5,5)$ & $71 \pm 3$ & $\mathbf{85 \pm 3}$$^{\dagger}$ & $76 \pm 4$ & $67 \pm 20$ & $58 \pm 23$ & +18 [+2, +44]* \\
\texttt{cube-single-noisy} & $(5,5)$ & $\mathbf{99 \pm 1}$ & $\mathbf{100 \pm 0}$$^{\dagger}$ & $\mathbf{100 \pm 0}$ & $\mathbf{100 \pm 0}$ & $\mathbf{100 \pm 0}$ & +0 [-1, +1] \\
\texttt{cube-double-play} & $(5,5)$ & $36 \pm 7$ & $\mathbf{69 \pm 3}$$^{\dagger}$ & $62 \pm 5$ & $51 \pm 5$ & $62 \pm 7$ & -0 [-5, +5] \\
\texttt{cube-double-noisy} & $(5,5)$ & $26 \pm 9$ & $\mathbf{60 \pm 7}$$^{\dagger}$ & $45 \pm 6$ & $38 \pm 3$ & $38 \pm 6$ & +7 [+3, +12]* \\
\texttt{cube-triple-play} & $(5,5)$ & $5 \pm 3$ & $1 \pm 1$$^{\dagger}$ & $\mathbf{9 \pm 2}$ & $3 \pm 2$ & $7 \pm 2$ & +2 [-1, +4] \\
\texttt{cube-triple-noisy} & $(5,5)$ & $9 \pm 3$ & $\mathbf{18 \pm 3}$$^{\dagger}$ & $10 \pm 4$ & $7 \pm 3$ & $8 \pm 3$ & +2 [-1, +5] \\
\texttt{scene-play} & $(10,10)$ & $52 \pm 5$ & $\mathbf{87 \pm 4}$$^{\dagger}$ & $68 \pm 4$ & $64 \pm 5$ & $61 \pm 5$ & +8 [+1, +13]* \\
\texttt{scene-noisy} & $(10,10)$ & $30 \pm 3$ & $\mathbf{70 \pm 5}$$^{\dagger}$ & $57 \pm 4$ & $41 \pm 5$ & $48 \pm 5$ & +9 [+5, +13]* \\
\texttt{puzzle-3x3-play} & $(10,10)$ & $91 \pm 4$ & $13 \pm 4$ & $\mathbf{100 \pm 0}$$^{\dagger}$ & $\mathbf{96 \pm 3}$ & $76 \pm 4$ & +24 [+20, +28]* \\
\texttt{puzzle-3x3-noisy} & $(10,10)$ & $\mathbf{91 \pm 3}$ & $41 \pm 7$ & $\mathbf{93 \pm 2}$$^{\dagger}$ & $88 \pm 3$ & $\mathbf{93 \pm 2}$ & -0 [-3, +2] \\
\texttt{puzzle-4x4-play} & $(10,10)$ & $24 \pm 5$ & $17 \pm 3$ & $\mathbf{51 \pm 8}$$^{\dagger}$ & $47 \pm 9$ & $26 \pm 5$ & +25 [+15, +37]* \\
\texttt{puzzle-4x4-noisy} & $(10,10)$ & $34 \pm 5$ & $43 \pm 5$ & $\mathbf{58 \pm 11}$$^{\dagger}$ & $53 \pm 7$ & $55 \pm 10$ & +3 [-1, +7] \\
\texttt{puzzle-4x5-play} & $(10,10)$ & $14 \pm 3$ & $10 \pm 3$ & $18 \pm 3$$^{\dagger}$ & $\mathbf{20 \pm 3}$ & $10 \pm 2$ & +8 [+5, +11]* \\
\texttt{puzzle-4x5-noisy} & $(10,10)$ & $\mathbf{20 \pm 3}$ & $\mathbf{20 \pm 3}$ & $\mathbf{20 \pm 3}$$^{\dagger}$ & $\mathbf{20 \pm 3}$ & $18 \pm 3$ & +2 [+1, +3]* \\
\texttt{puzzle-4x6-play} & $(10,10)$ & $11 \pm 2$ & $13 \pm 3$ & $17 \pm 3$$^{\dagger}$ & $\mathbf{20 \pm 3}$ & $7 \pm 2$ & +10 [+7, +13]* \\
\texttt{puzzle-4x6-noisy} & $(10,10)$ & $\mathbf{19 \pm 3}$ & $13 \pm 3$ & $\mathbf{20 \pm 3}$$^{\dagger}$ & $\mathbf{19 \pm 3}$ & $18 \pm 3$ & +2 [+1, +4]* \\
\bottomrule
\end{tabular}

%% file: tables/native_critic.tex
\begin{tabular}{lccc r}
\toprule
Dataset & $(k,c)$ & Best member & \wmpa{}-own & \wmpa{}-shared \\
\midrule
\texttt{pointmaze-medium-navigate} & $(1,1)$ & $\mathbf{74 \pm 3}$ & $54 \pm 16$ & $\mathbf{76 \pm 6}$ \\
\texttt{antmaze-large-navigate} & $(1,1)$ & $\mathbf{90 \pm 3}$ & $19 \pm 4$ & $\mathbf{90 \pm 3}$ \\
\texttt{cube-single-play} & $(5,5)$ & $71 \pm 3$ & $71 \pm 3$ & $\mathbf{81 \pm 4}$ \\
\texttt{cube-single-noisy} & $(5,5)$ & $\mathbf{99 \pm 1}$ & $94 \pm 4$ & $\mathbf{100 \pm 0}$ \\
\texttt{cube-double-play} & $(5,5)$ & $36 \pm 7$ & $37 \pm 7$ & $\mathbf{76 \pm 5}$ \\
\texttt{cube-double-noisy} & $(5,5)$ & $26 \pm 9$ & $25 \pm 5$ & $\mathbf{53 \pm 4}$ \\
\texttt{cube-triple-play} & $(5,5)$ & $\mathbf{2 \pm 1}$ & $\mathbf{2 \pm 2}$ & $\mathbf{2 \pm 1}$ \\
\texttt{cube-triple-noisy} & $(5,5)$ & $9 \pm 3$ & $3 \pm 2$ & $\mathbf{15 \pm 4}$ \\
\texttt{scene-play} & $(10,10)$ & $52 \pm 5$ & $52 \pm 5$ & $\mathbf{82 \pm 6}$ \\
\texttt{scene-noisy} & $(10,10)$ & $30 \pm 4$ & $36 \pm 4$ & $\mathbf{57 \pm 5}$ \\
\texttt{puzzle-3x3-play} & $(10,10)$ & $91 \pm 4$ & $94 \pm 3$ & $\mathbf{99 \pm 1}$ \\
\texttt{puzzle-3x3-noisy} & $(10,10)$ & $\mathbf{91 \pm 3}$ & $\mathbf{92 \pm 3}$ & $\mathbf{93 \pm 2}$ \\
\texttt{puzzle-4x4-play} & $(10,10)$ & $24 \pm 5$ & $26 \pm 5$ & $\mathbf{47 \pm 6}$ \\
\texttt{puzzle-4x4-noisy} & $(10,10)$ & $34 \pm 5$ & $41 \pm 8$ & $\mathbf{55 \pm 9}$ \\
\texttt{puzzle-4x5-play} & $(10,10)$ & $14 \pm 3$ & $14 \pm 3$ & $\mathbf{17 \pm 3}$ \\
\texttt{puzzle-4x5-noisy} & $(10,10)$ & $\mathbf{20 \pm 3}$ & $\mathbf{20 \pm 3}$ & $\mathbf{20 \pm 3}$ \\
\texttt{puzzle-4x6-play} & $(10,10)$ & $11 \pm 2$ & $13 \pm 2$ & $\mathbf{17 \pm 4}$ \\
\texttt{puzzle-4x6-noisy} & $(10,10)$ & $\mathbf{19 \pm 3}$ & $\mathbf{20 \pm 3}$ & $\mathbf{20 \pm 3}$ \\
\bottomrule
\end{tabular}

%% file: tables/value_head_selection.tex
\begin{tabular}{lccc c}
\toprule
Dataset & metric & direct & direct $-$ metric [95\% CI] & head in \cref{tab:main} \\
\midrule
\texttt{cube-double-play} & $\mathbf{0.75}$ & $0.63$ & $-0.12$ [$-0.13$, $-0.12$]* & metric \\
\texttt{scene-play} & $\mathbf{0.75}$ & $0.67$ & $-0.08$ [$-0.08$, $-0.08$]* & metric \\
\texttt{puzzle-4x4-play} & $0.70$ & $\mathbf{0.82}$ & $+0.12$ [$+0.11$, $+0.12$]* & direct \\
\texttt{cube-triple-play} & $0.64$ & $\mathbf{0.66}$ & $+0.02$ [$+0.01$, $+0.03$]* & metric \\
\bottomrule
\end{tabular}

%% file: tables/ablations.tex
\begin{tabular}{lccc|cccc}
\toprule
 & & & & \multicolumn{2}{c}{horizon agg.} & \multicolumn{2}{c}{ensemble agg.} \\
\cmidrule(lr){5-6}\cmidrule(lr){7-8}
Dataset & $(k,c)$ & Best & \wmpp & last & mean & min & mean-std \\
\midrule
\texttt{pointmaze-medium-navigate} & $(1,1)$ & $\mathbf{76 \pm 17}$ & $\mathbf{79 \pm 3}$ & $\mathbf{79 \pm 3}$ & $\mathbf{79 \pm 3}$ & $\mathbf{80 \pm 4}$ & $\mathbf{80 \pm 4}$ \\
\texttt{antmaze-large-navigate} & $(1,1)$ & $\mathbf{90 \pm 3}$ & $\mathbf{91 \pm 3}$ & $\mathbf{91 \pm 3}$ & $\mathbf{91 \pm 3}$ & $\mathbf{92 \pm 3}$ & $\mathbf{91 \pm 3}$ \\
\texttt{cube-single-play} & $(5,5)$ & $71 \pm 3$ & $\mathbf{85 \pm 3}$ & $\mathbf{87 \pm 3}$ & $\mathbf{84 \pm 4}$ & $\mathbf{88 \pm 3}$ & $\mathbf{87 \pm 3}$ \\
\texttt{cube-single-noisy} & $(5,5)$ & $\mathbf{99 \pm 1}$ & $\mathbf{100 \pm 0}$ & $\mathbf{100 \pm 0}$ & $\mathbf{100 \pm 0}$ & $\mathbf{100 \pm 0}$ & $\mathbf{99 \pm 1}$ \\
\texttt{cube-double-play} & $(5,5)$ & $36 \pm 7$ & $\mathbf{69 \pm 3}$ & $\mathbf{71 \pm 4}$ & $66 \pm 4$ & $\mathbf{71 \pm 5}$ & $\mathbf{70 \pm 4}$ \\
\texttt{cube-double-noisy} & $(5,5)$ & $26 \pm 9$ & $\mathbf{60 \pm 7}$ & $56 \pm 5$ & $54 \pm 4$ & $51 \pm 4$ & $56 \pm 4$ \\
\texttt{cube-triple-play} & $(5,5)$ & $\mathbf{5 \pm 3}$ & $1 \pm 1$ & $1 \pm 1$ & $0 \pm 1$ & $1 \pm 1$ & $0 \pm 0$ \\
\texttt{cube-triple-noisy} & $(5,5)$ & $9 \pm 3$ & $18 \pm 3$ & $17 \pm 3$ & $15 \pm 3$ & $\mathbf{19 \pm 3}$ & $\mathbf{19 \pm 4}$ \\
\texttt{scene-play} & $(10,10)$ & $52 \pm 5$ & $87 \pm 4$ & $\mathbf{91 \pm 3}$ & $\mathbf{92 \pm 4}$ & $\mathbf{89 \pm 3}$ & $\mathbf{90 \pm 3}$ \\
\texttt{scene-noisy} & $(10,10)$ & $30 \pm 3$ & $70 \pm 5$ & $72 \pm 5$ & $70 \pm 5$ & $\mathbf{75 \pm 5}$ & $\mathbf{78 \pm 5}$ \\
\texttt{puzzle-3x3-play} & $(10,10)$ & $91 \pm 4$ & $\mathbf{100 \pm 0}$ & $\mathbf{99 \pm 1}$ & $\mathbf{100 \pm 0}$ & $\mathbf{99 \pm 1}$ & $\mathbf{100 \pm 1}$ \\
\texttt{puzzle-3x3-noisy} & $(10,10)$ & $\mathbf{91 \pm 3}$ & $\mathbf{93 \pm 2}$ & $\mathbf{92 \pm 3}$ & $\mathbf{91 \pm 2}$ & $\mathbf{94 \pm 2}$ & $\mathbf{93 \pm 2}$ \\
\texttt{puzzle-4x4-play} & $(10,10)$ & $24 \pm 5$ & $51 \pm 8$ & $\mathbf{59 \pm 12}$ & $54 \pm 8$ & $53 \pm 7$ & $53 \pm 5$ \\
\texttt{puzzle-4x4-noisy} & $(10,10)$ & $34 \pm 5$ & $\mathbf{58 \pm 11}$ & $\mathbf{58 \pm 9}$ & $54 \pm 11$ & $\mathbf{56 \pm 12}$ & $\mathbf{55 \pm 12}$ \\
\texttt{puzzle-4x5-play} & $(10,10)$ & $14 \pm 3$ & $\mathbf{18 \pm 3}$ & $\mathbf{19 \pm 3}$ & $\mathbf{18 \pm 3}$ & $17 \pm 3$ & $17 \pm 3$ \\
\texttt{puzzle-4x5-noisy} & $(10,10)$ & $\mathbf{20 \pm 3}$ & $\mathbf{20 \pm 3}$ & $\mathbf{20 \pm 3}$ & $\mathbf{20 \pm 3}$ & $\mathbf{20 \pm 3}$ & $\mathbf{20 \pm 3}$ \\
\texttt{puzzle-4x6-play} & $(10,10)$ & $11 \pm 2$ & $17 \pm 3$ & $\mathbf{19 \pm 3}$ & $\mathbf{18 \pm 3}$ & $\mathbf{18 \pm 3}$ & $\mathbf{18 \pm 3}$ \\
\texttt{puzzle-4x6-noisy} & $(10,10)$ & $\mathbf{19 \pm 3}$ & $\mathbf{20 \pm 3}$ & $\mathbf{20 \pm 3}$ & $\mathbf{20 \pm 3}$ & $\mathbf{19 \pm 3}$ & $\mathbf{19 \pm 3}$ \\
\bottomrule
\end{tabular}

%% file: tables/bank_composition.tex
\begin{tabular}{llccccrr}
\toprule
Dataset & Bank & $P$ & Best member & \wmpa & $\Delta$ vs best member & $\Delta$ vs six-member \wmpa \\
\midrule
\multirow{2}{*}{\texttt{cube-double-play}} & Full bank & 6 & GCIQL ($36 \pm 7$) & $\mathbf{69 \pm 3}$ &  &  \\
 & Pooled play+noisy & 12 & GCIQL ($36 \pm 7$) & $\mathbf{79 \pm 3}$ & +43 [+35, +50]* & +10 [+6, +15]* \\
\midrule
\multirow{2}{*}{\texttt{scene-play}} & Full bank & 6 & GCIQL ($52 \pm 5$) & $\mathbf{87 \pm 4}$ &  &  \\
 & Pooled play+noisy & 12 & GCIQL ($52 \pm 5$) & $\mathbf{88 \pm 3}$ & +37 [+30, +43]* & +1 [-3, +5] \\
\midrule
\multirow{2}{*}{\texttt{puzzle-4x4-play}} & Full bank & 6 & GCIQL ($24 \pm 5$) & $\mathbf{51 \pm 8}$ &  &  \\
 & Pooled play+noisy & 12 & GCIQL (noisy) ($34 \pm 4$) & $\mathbf{53 \pm 6}$ & +18 [+11, +26]* & +2 [-4, +8] \\
\bottomrule
\end{tabular}

%% file: tables/bank_stats.tex
\begin{tabular}{lccccc rr}
\toprule
Dataset & $B_{\max}$ & $D_{\mathrm{gap}}$ & $U$ (oracle) & $U-B_{\max}$ & \wmpp & $\Delta$\wmpp{} [95\% CI] & \wmpp$-U$ [95\% CI] \\
\midrule
\texttt{pointmaze-medium-navigate} & $76 \pm 17$ & $2$ & $\mathbf{98 \pm 1}$ & $22$ & $79 \pm 3$ & +3 [-17, +17] & -19 [-22, -16]* \\
\texttt{antmaze-large-navigate} & $90 \pm 3$ & $4$ & $\mathbf{99 \pm 1}$ & $9$ & $91 \pm 3$ & +0 [-3, +4] & -9 [-12, -6]* \\
\texttt{cube-single-play} & $71 \pm 3$ & $22$ & $\mathbf{81 \pm 3}$ & $10$ & $\mathbf{85 \pm 3}$ & +14 [+10, +18]* & +4 [-1, +9] \\
\texttt{cube-single-noisy} & $\mathbf{99 \pm 1}$ & $28$ & $\mathbf{100 \pm 0}$ & $1$ & $\mathbf{100 \pm 0}$ & +1 [+0, +2]* & +0 [-0, +1] \\
\texttt{cube-double-play} & $36 \pm 7$ & $1$ & $57 \pm 6$ & $21$ & $\mathbf{69 \pm 3}$ & +33 [+25, +40]* & +12 [+6, +18]* \\
\texttt{cube-double-noisy} & $26 \pm 9$ & $11$ & $31 \pm 6$ & $6$ & $\mathbf{60 \pm 7}$ & +35 [+30, +40]* & +29 [+24, +34]* \\
\texttt{cube-triple-play} & $5 \pm 3$ & $2$ & $\mathbf{9 \pm 3}$ & $4$ & $1 \pm 1$ & -4 [-8, -1]* & -8 [-11, -5]* \\
\texttt{cube-triple-noisy} & $9 \pm 3$ & $7$ & $11 \pm 3$ & $2$ & $\mathbf{18 \pm 3}$ & +9 [+6, +12]* & +7 [+3, +10]* \\
\texttt{scene-play} & $52 \pm 5$ & $9$ & $74 \pm 6$ & $22$ & $\mathbf{87 \pm 4}$ & +36 [+30, +41]* & +13 [+8, +19]* \\
\texttt{scene-noisy} & $30 \pm 3$ & $3$ & $52 \pm 4$ & $22$ & $\mathbf{70 \pm 5}$ & +41 [+36, +45]* & +19 [+14, +23]* \\
\texttt{puzzle-3x3-play} & $91 \pm 4$ & $79$ & $92 \pm 3$ & $1$ & $\mathbf{100 \pm 0}$ & +9 [+5, +13]* & +8 [+5, +11]* \\
\texttt{puzzle-3x3-noisy} & $91 \pm 3$ & $34$ & $\mathbf{97 \pm 2}$ & $6$ & $\mathbf{93 \pm 2}$ & +1 [-1, +4] & -5 [-7, -2]* \\
\texttt{puzzle-4x4-play} & $24 \pm 5$ & $13$ & $37 \pm 5$ & $13$ & $\mathbf{51 \pm 8}$ & +27 [+18, +36]* & +14 [+7, +21]* \\
\texttt{puzzle-4x4-noisy} & $34 \pm 5$ & $15$ & $52 \pm 7$ & $18$ & $\mathbf{58 \pm 11}$ & +24 [+16, +32]* & +6 [-2, +17] \\
\texttt{puzzle-4x5-play} & $14 \pm 3$ & $9$ & $17 \pm 3$ & $3$ & $\mathbf{18 \pm 3}$ & +4 [+3, +6]* & +1 [-0, +3] \\
\texttt{puzzle-4x5-noisy} & $\mathbf{20 \pm 3}$ & $0$ & $\mathbf{20 \pm 3}$ & $0$ & $\mathbf{20 \pm 3}$ & +0 [+0, +1] & -0 [-1, +0] \\
\texttt{puzzle-4x6-play} & $11 \pm 2$ & $0$ & $\mathbf{19 \pm 3}$ & $8$ & $17 \pm 3$ & +6 [+3, +9]* & -2 [-6, +2] \\
\texttt{puzzle-4x6-noisy} & $19 \pm 3$ & $1$ & $\mathbf{22 \pm 3}$ & $3$ & $20 \pm 3$ & +1 [-0, +2] & -2 [-3, -1]* \\
\bottomrule
\end{tabular}

%% file: tables/runtime.tex
\begin{tabular}{lcccccccccc}
\toprule
 & & & \multicolumn{2}{c}{switches/ep} & & & & \multicolumn{3}{c}{ms/step (CPU)} \\
\cmidrule(lr){4-5}\cmidrule(lr){9-11}
Dataset & $(k,c)$ & ep.\ len & \wmpp & Rand. & decisions/ep & WM calls/step & usage $H$ (bits) & \wmpp & Rand. & fixed \\
\midrule
\texttt{pointmaze-medium-navigate} & $(1,1)$ & 295 & 24.8 & 246.5 & 294.8 & 18.0 & 1.30 & 3.76 & 0.21 & 0.18 \\
\texttt{antmaze-large-navigate} & $(1,1)$ & 475 & 335.9 & 425.8 & 475.2 & 18.0 & 2.57 & 3.88 & 0.19 & 0.15 \\
\texttt{cube-single-play} & $(5,5)$ & 90 & 13.0 & 27.3 & 18.3 & 18.3 & 2.54 & 3.29 & 0.35 & 0.23 \\
\texttt{cube-single-noisy} & $(5,5)$ & 46 & 3.7 & 13.0 & 9.5 & 18.8 & 1.61 & 3.24 & 0.24 & 0.30 \\
\texttt{cube-double-play} & $(5,5)$ & 264 & 37.8 & 73.8 & 53.1 & 18.1 & 2.55 & 2.97 & 0.23 & 0.22 \\
\texttt{cube-double-noisy} & $(5,5)$ & 310 & 41.4 & 69.9 & 62.3 & 18.1 & 2.45 & 2.89 & 0.23 & 0.34 \\
\texttt{cube-triple-play} & $(5,5)$ & 996 & 127.1 & 161.0 & 199.1 & 18.0 & 2.51 & 2.77 & 0.25 & 0.21 \\
\texttt{cube-triple-noisy} & $(5,5)$ & 870 & 103.9 & 139.9 & 174.1 & 18.0 & 2.54 & 2.96 & 0.26 & 0.26 \\
\texttt{scene-play} & $(10,10)$ & 277 & 20.7 & 48.7 & 28.1 & 18.3 & 2.48 & 2.98 & 0.27 & 0.24 \\
\texttt{scene-noisy} & $(10,10)$ & 355 & 20.7 & 44.4 & 35.8 & 18.2 & 2.40 & 2.60 & 0.23 & 0.19 \\
\texttt{puzzle-3x3-play} & $(10,10)$ & 195 & 15.8 & 36.9 & 20.0 & 18.4 & 2.50 & 3.85 & 0.25 & 0.23 \\
\texttt{puzzle-3x3-noisy} & $(10,10)$ & 102 & 3.3 & 18.0 & 10.7 & 18.7 & 1.37 & 4.26 & 0.26 & 0.27 \\
\texttt{puzzle-4x4-play} & $(10,10)$ & 361 & 29.1 & 40.0 & 36.3 & 18.1 & 2.52 & 4.46 & 0.26 & 0.24 \\
\texttt{puzzle-4x4-noisy} & $(10,10)$ & 281 & 16.1 & 32.0 & 28.3 & 18.2 & 2.09 & 3.50 & 0.40 & 0.20 \\
\texttt{puzzle-4x5-play} & $(10,10)$ & 840 & 61.3 & 80.5 & 84.1 & 18.0 & 2.34 & 4.19 & 0.27 & 0.29 \\
\texttt{puzzle-4x5-noisy} & $(10,10)$ & 811 & 41.2 & 73.7 & 81.2 & 18.0 & 2.09 & 4.13 & 0.25 & 0.21 \\
\texttt{puzzle-4x6-play} & $(10,10)$ & 858 & 66.1 & 80.2 & 85.9 & 18.0 & 2.44 & 3.82 & 0.45 & 0.34 \\
\texttt{puzzle-4x6-noisy} & $(10,10)$ & 819 & 43.9 & 72.8 & 82.0 & 18.0 & 2.08 & 4.38 & 0.27 & 0.41 \\
\bottomrule
\end{tabular}